\documentclass{article}

\usepackage[preprint]{icml2026}

\usepackage{microtype}
\usepackage{graphicx}
\usepackage{subcaption}
\usepackage{booktabs}
\usepackage{amsmath,amssymb,mathtools}
\usepackage{amsthm}
\usepackage{hyperref}
\usepackage[capitalize,noabbrev]{cleveref}
\usepackage{multirow}
\usepackage{makecell} 
\usepackage{placeins} % 在导言区添加
\usepackage[textsize=tiny]{todonotes}
\usepackage{algorithm}
\usepackage{algorithmic}
\icmltitlerunning{NECromancer: Breathing Life into Skeletons via BVH Animation}

\newcommand{\dataname}{Unified BVH Universe}
\newtheorem*{theorem*}{Theorem}

\begin{document}

\twocolumn[
\icmltitle{NECromancer: Breathing Life into Skeletons via BVH Animation}
\icmlsetsymbol{equal}{*}

\begin{icmlauthorlist}
  \icmlauthor{Mingxi Xu}{huawei1}
  \icmlauthor{Qi Wang}{huawei1}
  \icmlauthor{Zhengyu Wen}{huawei1}
  \icmlauthor{Phong Dao Thien}{huawei1}
  \icmlauthor{Zhengyu Li}{huawei1}
  \icmlauthor{Ning Zhang}{huawei1}
  \icmlauthor{Xiaoyu He}{huawei1}
  \icmlauthor{Wei Zhao}{huawei1}
  \icmlauthor{Kehong Gong}{huawei2}
  \icmlauthor{Mingyuan Zhang}{huawei1}
\end{icmlauthorlist}

\icmlaffiliation{huawei1}{Huawei Central Media Technology Institute}
\icmlaffiliation{huawei2}{Huawei Technologies Co., Ltd.} 
\icmlcorrespondingauthor{Mingyuan Zhang}{zhangmy718@gmail.com}

  \icmlkeywords{Machine Learning, ICML}
  \vskip 0.2in 

  % ======= 简化Teaser插入 =======
  \centering
  \icmlkeywords{Machine Learning, ICML}

]

\printAffiliationsAndNotice{}

\begin{abstract}
Motion tokenization is a key component of generalizable motion models, yet most existing approaches are restricted to species-specific skeletons, limiting their applicability across diverse morphologies. We propose \textbf{NECromancer (NEC)}, a universal motion tokenizer that operates directly on arbitrary BVH skeletons. NEC consists of three components: (1) an \textbf{Ontology-aware Skeletal Graph Encoder (OwO)} that encodes structural priors from BVH files—including joint semantics, rest-pose offsets, and skeletal topology—into skeletal embeddings; (2) a \textbf{Topology-Agnostic Tokenizer (TAT)} that compresses motion sequences into a universal, topology-invariant discrete representation; and (3) the \textbf{Unified BVH Universe (UvU)}, a large-scale dataset aggregating BVH motions across heterogeneous skeletons. Experiments show that NEC achieves high-fidelity reconstruction under substantial compression and effectively disentangles motion from skeletal structure. The resulting token space supports cross-species motion transfer, composition, denoising, generation with token-based models, and text-motion retrieval, establishing a unified framework for motion analysis and synthesis across diverse morphologies.
Demo page: \url{https://animotionlab.github.io/NECromancer/}
\end{abstract} 
\vspace{-20pt}
\section{Introduction}
\label{sec:intro}

Generating dynamic 4D content is a core capability of world models operating in complex environments. While deformation-field methods extend static 3D representations with temporal dynamics, they often suffer from geometric inconsistency, high computational cost, and limited controllability~\cite{li2024dreammesh4d}. Moreover, many pipelines reconstruct 4D content from videos or video generative priors~\cite{cao2024avatargo,ren2024l4gm}, which limits the learning of structured and generalizable motion representations.

Skeleton-based modeling instead provides an interpretable, compact, and controllable abstraction for motion and has driven strong progress in human-centric 4D content~\cite{hong2022avatarclip,tevet2023human}. However, most existing methods do not generalize well across embodiments. Approaches based on fixed human templates~\cite{tevet2023human,zhang2024motiondiffuse} or canonical skeleton retargeting~\cite{wang2025animo} restrict expressivity for diverse body plans, while keypoint-only representations remain misaligned with modern CG pipelines~\cite{guo2022generating,plappert2016kit}. These assumptions of near-static skeletal topology fundamentally limit multi-species animation and universal motion understanding.

We therefore target \emph{topology-invariant yet semantics-preserving} motion representations that can accommodate diverse joint hierarchies and spatio-temporal relations. To this end, we introduce an \emph{Ontology-aware Skeletal Graph Encoder} (OwO) that maps rest-pose skeletons to per-joint structural embeddings capturing topological, postural, and semantic cues. Conditioned on these embeddings, we develop a \emph{Topology-Agnostic Tokenizer} that converts BVH motion sequences on arbitrary skeletons into compact, discrete tokens compatible with token-based generators.

To support learning and evaluation under diverse morphologies, we curate a BVH-centric benchmark by consolidating HumanML3D~\cite{guo2022generating}, Objaverse-XL~\cite{objaverseXL}, and Truebones Zoo~\cite{truebones_mocap} through extensive cleaning and normalization. The resulting dataset, \dataname, contains 47,807 high-quality text-annotated motion sequences spanning humans, quadrupeds, and other species.

Under a unified evaluation protocol, our tokenizer achieves strong reconstruction with substantial compression and outperforms RVQVAE baselines in retrieval accuracy, generation fidelity, and joint-space error. Beyond reconstruction, the learned token space natively supports \emph{any-skeleton} motion generation, transfer, and retrieval, while the OwO structural prior further enables topology-agnostic Motion–Text alignment.

Our contributions are threefold:
\begin{enumerate}
    \item \textbf{Ontology-aware Skeletal Graph Encoder.}
    We introduce a graph-based skeleton embedder with self-supervised objectives that capture topological, postural, and semantic structure, producing reusable joint-level representations for BVH-format data.

    \item \textbf{Topology-Agnostic Tokenizer.}
    To operate on arbitrary BVH skeletons and produces compact, token-level, skeleton-agnostic representations, we develop a graph-conditioned motion tokenizer. This design enables a variety of additional appealing applications once the tokenizer has been trained.

    \item \textbf{BVH-centric Benchmark.}
    We establish a large-scale curated BVH benchmark (47{,}807 sequences) spanning heterogeneous species and skeletal structures, supporting standardized evaluation of topology generalization for reconstruction, retrieval (R-Precision@K), and distributional quality (FID).
\end{enumerate}
\section{Related Works}

\label{sec:related_works}

\paragraph{Skeleton-based Motion Generation.}
Skeletons provide a compact, interpretable, and controllable abstraction for motion and have seen rapid progress under text-, audio-, and scene-conditioned settings. Text-to-motion methods evolved from embedding or variational alignment (e.g., MotionCLIP, TEMOS) to diffusion-based backbones enabling higher fidelity and local editing~\cite{tevet2022motionclip,petrovich2022temos,tevet2023human,kim2023flame}. Controllability and efficiency were further improved via multi-level conditioning and retrieval augmentation (MotionDiffuse, ReMoDiffuse), parameter-efficient adaptation (LoRA-MDM), and latent or autoregressive decoding (MLD, DART)~\cite{zhang2024motiondiffuse,zhang2023remodiffuse,chen2023executing,Zhao:DartControl:2025}. Audio-conditioned gesture and dance generation followed a similar trajectory, progressing from deterministic or VAE-based mappings to diffusion- or transformer-based models with rhythm or key-pose guidance~\cite{li2021audio2gestures,Ao2023GestureDiffuCLIP,siyao2022bailando,huang2025beat}. Despite strong performance, most systems assume a fixed human topology, limiting transfer, retargeting, and generalization to non-human embodiments.

\paragraph{Discrete Motion Tokenization.}
A parallel line of work discretizes motion into compact token sequences to enable LLM-style generation, editing, and retrieval. Vector-quantized autoencoders~\cite{van2017neural}, including residual and hierarchical variants, as well as masked modeling, have been widely used to build motion codebooks~\cite{guo2023momask} and unify multiple tasks within a single interface~\cite{jiang2024motiongpt}. Compared to purely continuous diffusion models, discrete representations better support long-horizon reasoning, scalable data mixing, and plug-and-play conditioning with language or audio. However, existing tokenizers are typically tied to a canonical human skeleton and fixed joint definitions, limiting their use as a universal motion representation~\cite{guo2022tm2t}. In contrast, we condition tokenization on an \emph{ontology-aware skeletal graph} derived from the rest pose, enabling topology-agnostic BVH motion codes that generalize across arbitrary skeleton templates.

\paragraph{Cross-Embodiment, Retargeting, and BVH-centric Corpora.}
Cross-embodiment motion is commonly addressed by mapping motions to canonical templates (e.g., SMPL, GHUM) or via skeleton-aware retargeting networks~\cite{loper2023smpl,xu2020ghum,aberman2020skeleton}. While effective for human avatars, such normalization reduces expressivity for disparate body plans and bone-length statistics. Recent work has begun exploring arbitrary-topology and animal motion generation~\cite{gat2025anytop,wang2025animo}, yet a universal tokenizer that natively supports arbitrary BVH templates remains underexplored. Given BVH’s prevalence in motion capture and digital content creation pipelines, adopting BVH as a unifying representation facilitates large-scale aggregation and standardized evaluation. Accordingly, we release a BVH-centric corpus that consolidates heterogeneous sources~\cite{guo2022generating,objaverse,objaverseXL,truebones_mocap} and evaluate generalization under seen and unseen topologies and species.
% ===================== UVU =====================
\section{UvU: Unified BVH Universe}
\label{sec:dataset}

\subsection{Dataset Overview}
\label{sec:data_overview}
The \dataname \ dataset is designed to enable generalized motion understanding and synthesis across varying skeletal topologies. Unlike existing datasets that focus on a single skeletal structure or similar structures (e.g., human-only~\cite{loper2023smpl} or species-specific motions~\cite{biggs2020wldo}), our dataset introduces highly divergent skeletons like fantasy creatures, enabling generative models with the ability to drive a variety of skeletons.

The data format we are using within \dataname, as shown in Fig.~\ref{fig:dataset_pipe}, (2) representation part, consists of 4 parts: \textbf{BVH Motion Data} for temporal joint rotations, \textbf{Base Mesh} for 3D mesh template, \textbf{Skin Weights} for mesh deformation and \textbf{Text Annotations}. In summary, our dataset consists of 47,807 animations.

\begin{figure*}[t]
  \centering
  \includegraphics[width=\textwidth]{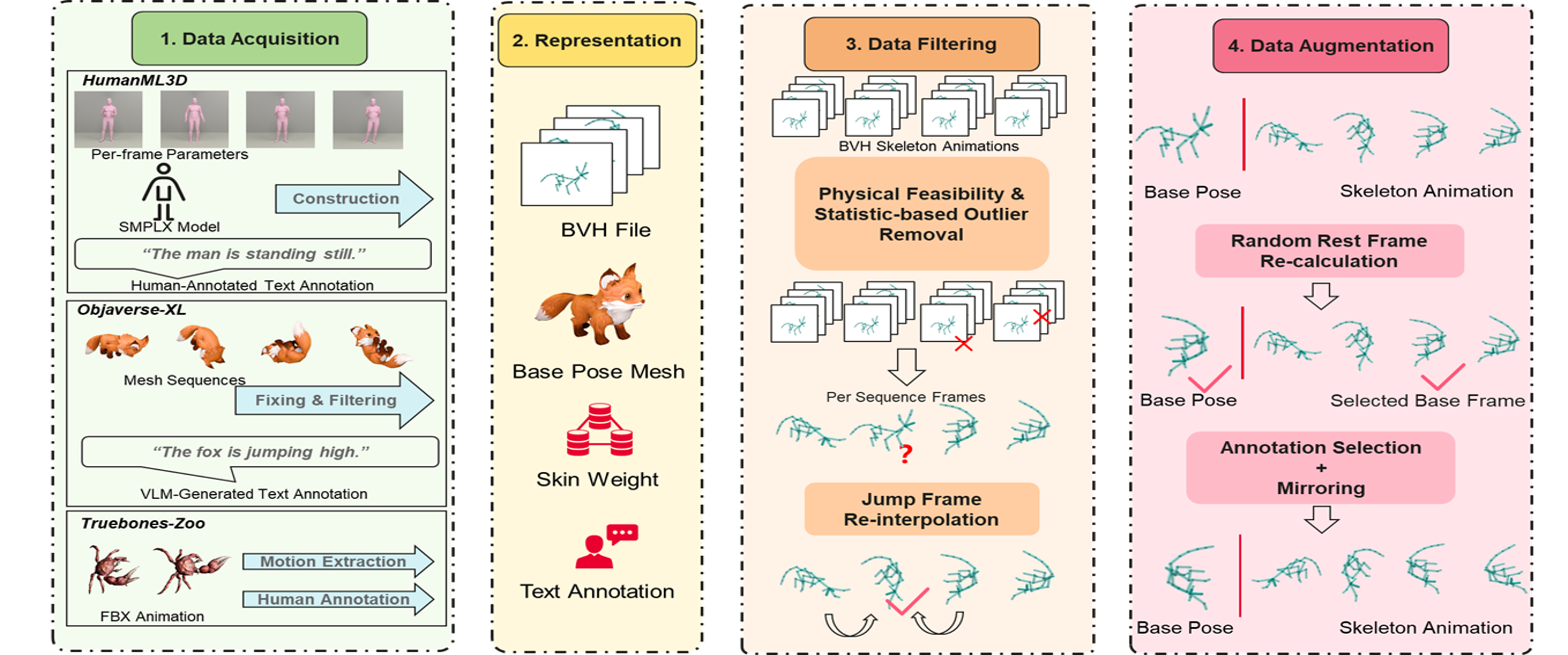}
  \caption{
  Overview of the \dataname\ dataset pipeline.
  Motion data from three existing datasets are unified into a standardized representation, including BVH files, base-pose meshes, skinning weights, and text annotations.
  Data filtering and smoothing are applied to ensure physical plausibility.
  During training, on-the-fly augmentations are used to further increase data diversity.
  }
  \label{fig:dataset_pipe}
\end{figure*}

\subsection{Data Preprocessing}
\label{sec:data_preprocess}

As shown in Fig.~\ref{fig:dataset_pipe}, to construct the \dataname \ dataset, we unify motion data from three heterogeneous sources: HumanML3D \cite{guo2022generating}, Objaverse-XL \cite{objaverseXL}, and Truebones Zoo \cite{truebones_mocap}. For HumanML3D, we reconstruct animations from SMPLX \cite{SMPL-X:2019} parameters, standardize skeleton structures, and convert them into BVH format. For Objaverse-XL, extensive filtering and correction are conducted to acquire high quality BVH animations. Through such preprocessing, we ensure that all animation data have only a single skeleton tree, global translations are aggregated on the root joint, and remove redundant or meaningless joints. Finally, Qwen2.5-VL \citep{bai2025qwen25vl} is utilized to further filter semantically acceptable motions and generate text annotations for them. Truebones Zoo provides diverse artist-animated FBX motions; we extract skeletal structures and animations, apply standardization, and supplement with human-annotated text descriptions. 
Data filtering and corrections are applied to improve consistency and quality. After preprocessing, we split each dataset separately. For HumanML3D, we use the original split \cite{guo2022generating}. For Objaverse-XL, we randomly divide the sequences into 85\% train and 15\% test. For Truebones Zoo, we ensure that at least one animal per category (e.g. biped, quadruped) is included, with 15\% of sequences as test and the rest as train. Detailed preprocessing procedures are provided in the supplementary material.
\section{Methods}
\label{sec:methods}

To support arbitrary skeletal topologies, we introduce a standalone Graph Embedder that encodes the rest-pose skeleton into joint-level identity embeddings. 
These embeddings condition both the encoder and decoder, enabling topology-aware tokenization with minimal overhead, as the graph is encoded once per skeleton.

An overview of the tokenizer structure is shown in Fig.~\ref{fig:tokenizer}.
We next describe the architecture of the Graph Embedder and its integration into the encoder and decoder.

\subsection{Problem Definition}
\label{sec:prob_def}

We unify all animation data into the BVH motion format during preprocessing.
An animation sequence with $J$ joints and $T$ frames is represented as
$\Theta \in \mathbb{R}^{T \times J \times 9}$.

For the root joint ($j=0$), the representation is
\begin{equation}
\Theta_{t,0} =
\begin{bmatrix}
\Delta x_t & \Delta y_t & \Delta z_t & \text{rot6D}_{t}^{(0)}
\end{bmatrix},
\end{equation}
where $(\Delta x_t, \Delta y_t, \Delta z_t)$ denotes the displacement relative to the previous frame along the three coordinate axes, and $\text{rot6D}_{t}^{(0)}$ denotes the global orientation in rot6D format~\citep{zhou2019continuity}.
For non-root joints ($j \neq 0$), the representation is
\begin{equation}
\Theta_{t,j} =
\begin{bmatrix}
0 & 0 & 0 & \text{rot6D}_{t}^{(j)}
\end{bmatrix},
\end{equation}
where $\text{rot6D}_{t}^{(j)}$ represents the rotation relative to the parent joint in the BVH kinematic tree, also expressed in rot6D format.

All joint rotations are defined relative to the BVH rest pose. Consequently, encoding and decoding BVH motions require access to the corresponding rest pose, which specifies a kinematic tree
$\mathcal{S} = (\mathcal{J}, \mathcal{E})$ encoding parent--child relationships among joints, fixed joint offsets $\mathbf{o}_{i \rightarrow j} \in \mathbb{R}^3$ from parent joint $i$ to joint $j$, as well as the name of each joint.

\subsection{OwO: Ontology-aware Skeletal Graph Encoder}
\label{sec:owo}

Since our tokenizer needs to handle different skeletal structures with varying numbers of joints, 
we aim to design a unified modeling scheme. 
Specifically, we require the tokenizer to transform the BVH motion $\Theta$ into a latent code
\[
\mathbf{z} \in \{1,2,\dots,K\}^{\lfloor \tfrac{T}{r} \rfloor \times R},
\]
where $r$ denotes the temporal compression ratio, $R$ is the number of residual tokens used in Residual Vector Quantization (RVQ) to represent the same continuous latent feature, 
and $K$ is the size of the codebook.

This formulation implies that, in the encoder stage of the tokenizer, spatial information across all joints within each frame must be effectively fused. 
In the decoder stage, the latent features must be able to accurately reconstruct the motion information of each joint. 
Therefore, we design a encoder that leverages the full information of the rest pose to extract a unique feature for each joint, serving as its identity embedding. To effectively model the topological structure among joints in the BVH, 

we construct a graph based on the rest pose information, 
and build several graph attention blocks on top of it for feature extraction. The detailed computational formulas are provided in the appendix.

Since 4D data is extremely scarce, we design a pre-training stage to better train the graph encoder, 
which can be applied to arbitrary rigged 3D models. Specifically, we design a set of self-supervised objectives to guide the training of the Graph Embedder. These tasks are crafted to encourage the node and global features to encode three critical aspects of skeletal structure: \textbf{geometric}, \textbf{topological}, and \textbf{semantic} information. The module is pretrained independently and frozen during downstream motion generation to serve as a transferable structural prior.

\begin{figure*}[t]
  \centering
  \includegraphics[width=\linewidth, keepaspectratio]{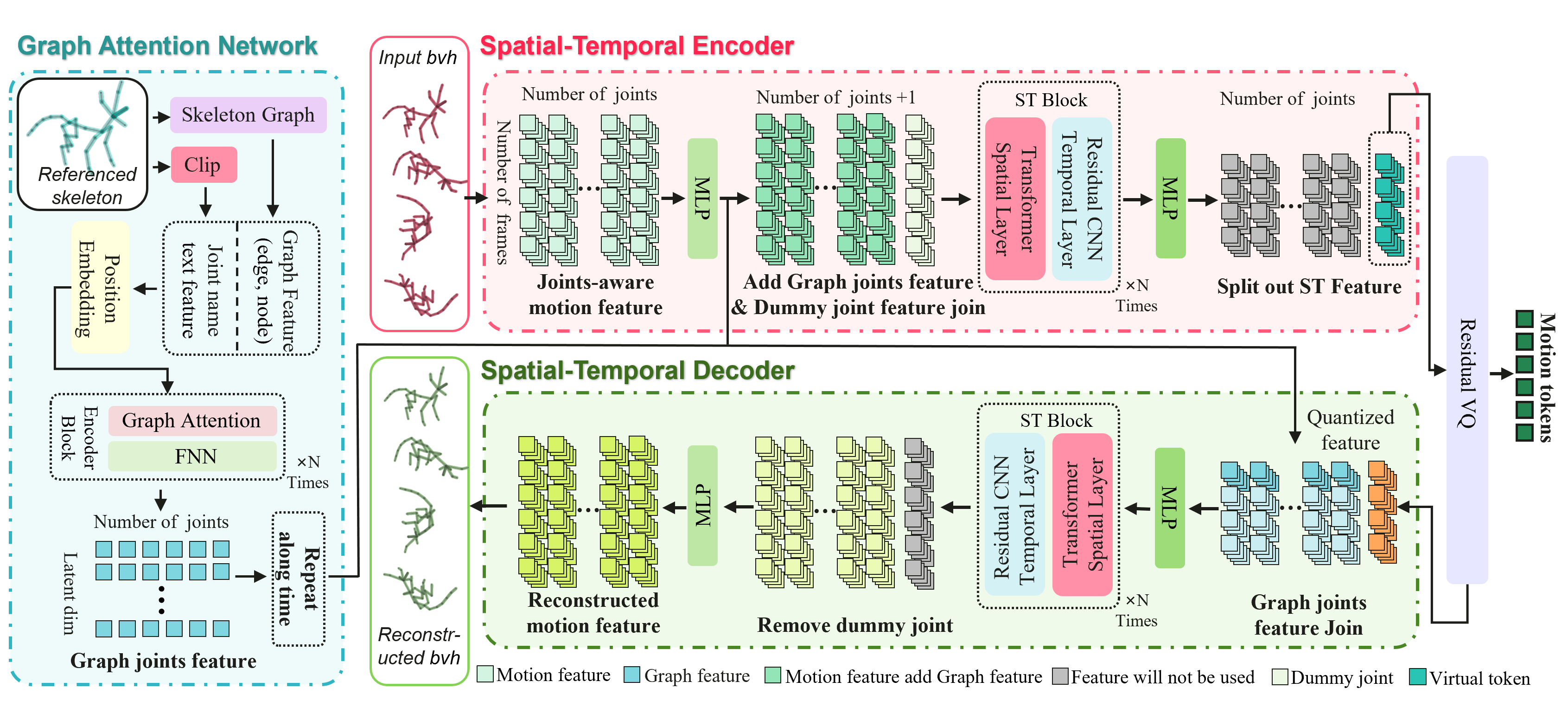}
  \caption{\textbf{Overview of NECromancer (NEC).} NEC consists of two main components: (a) Ontology-aware Skeletal Graph Encoder (OwO), which encodes static skeletal information (topology, joint names, rest pose) into structured graph-based joint features;(b) Topology-Agnostic Tokenizer (TAT), including Spatio-Temporal Encoder and Decoder, which maps motion sequences into a unified feature space, appends virtual joints, and converts them into discrete motion tokens.}
  \label{fig:tokenizer}
\end{figure*}
For the extracted node features, we design the following three types of loss functions to encourage them to capture different aspects of information:  

\begin{itemize}
  \item \textbf{Geometric Loss.}  
  This loss focuses on the recovery ability of the rest pose.  
  For any pair of joints $(i, j)$ information and their corresponding node features, a task-specific prediction head is required to output the offset of joint $j$ relative to joint $i$.  
  If the model can accurately solve this task, then the relative positions of all joints in the rest pose can be recovered.  

  \item \textbf{Topological Loss.}  
  This loss targets the connectivity information encoded in the kinematic tree of the original BVH.  
  Leveraging the tree structure, we require the model to correctly identify the lowest common ancestor (LCA) of any joint pair $(i, j)$ after the task-specific prediction head.  

  \begin{theorem*}
  If a model can correctly determine the LCA for any pair of nodes $(i, j)$ in a tree, then the entire tree topology can be uniquely reconstructed.
  \end{theorem*}

  \item \textbf{Semantic Loss.}  
  This loss encourages alignment between node features and semantic information.  
  We extract textual features of each joint name using CLIP, and apply a contrastive learning objective such that each joint’s node feature is pulled closer to its own name embedding, while being pushed away from those of other joints.
\end{itemize}

The detailed loss calculation and the proof of theorem are thoroughly introduced in the appendix.

\paragraph{Role of OwO within the Tokenizer.}
OwO is not an auxiliary component but the structural prior that enables topology-agnostic motion tokenization. 
Given a skeleton, OwO produces a set of joint-level identity embeddings
\[
F_{\text{node}}=\{h_j \in \mathbb{R}^d \}_{j=1}^J,
\]
which encode the semantic, geometric, and topological roles of all joints.

During TAT encoding, these structural embeddings are fused with the per-joint motion features:
\[
X_{t,j} = \mathrm{MLP}(\Theta_{t,j}) + \mathrm{Proj}(h_j),
\]
ensuring that the tokenizer interprets motions with respect to the correct anatomical meanings. 
This fusion allows the model to consistently understand motion patterns such as \emph{arm lifting}, 
\emph{spine bending}, or \emph{wing flapping}, even when applied to skeletons with different numbers of joints 
or distinct hierarchical structures. 

Importantly, OwO also plays a crucial role during decoding.  
Given a target skeleton, its OwO embeddings $F_{\text{node}}$ are repeated across the temporal dimension 
and concatenated to the quantized latent sequence:
\[
\tilde{Z}_{t,j} = [\, z_t \,\|\, h_j \,],
\]
where $z_t$ is the virtual-joint latent token at timestep $t$.  
This provides the decoder with a skeleton-specific \emph{template}, enabling it to reconstruct per-joint rotations 
consistent with the target morphology while preserving the motion dynamics encoded in the token sequence.  
Since OwO is computed once per skeleton and reused throughout the entire TAT pipeline, 
it serves as a lightweight yet expressive structural descriptor for universal motion reconstruction.

% 这怎么又讲一边，上面说过了感觉可以略过，除非打算介绍下静态数据
% \subsection{Experimental Setup}
% \label{sec:experiments_setup}
% \paragraph{Datasets and splits.}
% We evaluate on three BVH-centric sources: \textbf{HumanML3D}~\cite{guo2022generating}, \textbf{Objaverse-XL}~\cite{objaverseXL}, and \textbf{Truebones Zoo}~\cite{truebones_mocap}.
% Unless otherwise noted, we train \emph{one} model on the \emph{union} of training partitions from all three sources (\emph{UvU-train}) and evaluate on the \emph{union} of their test partitions (\emph{UvU-test}); per-dataset numbers are then computed by restricting evaluation to the corresponding source’s test subset. All methods (baselines and ours) are trained and evaluated under this unified protocol.

\subsection{TAT: Topology-Agnostic Tokenizer}
\label{sec:TAT}

\paragraph{Motivation: Why Topology-Agnostic Tokenization Is Challenging.}
Conventional VQ-based motion tokenizers quantize the feature at every joint and timestep, 
which implicitly assumes a fixed skeleton layout. 
This requirement fundamentally prevents them from handling heterogeneous skeletons with 
different joint counts and joint orderings, such as humans (22 joints), dogs (87 joints), birds (with folding wings), or dragons (over 120 joints).

To support motion tokenization across arbitrary skeletons, the quantization stage must be 
decoupled from the number of joints. 
Therefore, instead of quantizing joint-wise features, we introduce a \emph{virtual joint} 
that summarizes all joint features at each timestep into a topology-invariant representation. 
This design removes the dependence on the underlying kinematic structure and enables a truly universal discrete motion space.

\paragraph{Virtual Joint for Topology-Invariant Quantization.}
Instead of aggregating joint features via pooling, we introduce a learnable \emph{virtual joint} token, 
analogous to a classification token in Transformer encoders. 
At each (downsampled) timestep $t$, we augment the set of joint features 
$\{X_{t,j}\}_{j=1}^J$ with a virtual joint embedding $v_t^{(0)} \in \mathbb{R}^d$:
\[
\tilde{X}_{t} = \{X_{t,1}, \dots, X_{t,J}, v_t^{(0)}\}.
\]
This extended sequence is processed by $L$ stacked spatio-temporal blocks. 
Through the spatial attention, the virtual joint attends to all real joints and 
gradually accumulates a global summary of the motion at timestep $t$. 
We denote the virtual joint after the last block as $v_t^{(L)}$.

Crucially, only the virtual joint is fed into the RVQ quantizer:
\[
z_t = \mathrm{RVQ}(v_t^{(L)}),
\]
which yields a fixed-size discrete latent code independently of the number of real joints. 
During decoding, the quantized latent $z_t$ is injected back as the virtual joint token and combined with 
the OwO-conditioned joint features of the target skeleton to reconstruct per-joint rotations. 
This CLS-style virtual joint design enables topology-invariant motion tokenization without ever 
requiring a fixed joint grid.

Thanks to the explicitly extracted skeletal structure information from the Graph Embedder, we can simplify the motion reconstruction process under arbitrary topologies. Specifically, the inputs to this reconstruction step include:
\begin{itemize}
    \item The node-level structural embeddings from the Graph Embedder, denoted as
  $
  \mathbf{F}_{\text{node}} = \{\mathbf{h}_j \in \mathbb{R}^d\}_{j=1}^J,
  $
  where $J$ is the number of joints and $d$ is the embedding dimension;
    \item A motion sequence in our defined format, represented as
  $
  \mathbf{\Theta} \in \mathbb{R}^{T \times J \times 9},
  $
  where $T$ is the number of frames and each 9D vector encodes joint translation/rotation information as defined in  Section~\ref{sec:prob_def}.
\end{itemize}
By combining $\mathbf{F}_{\text{node}}$ with the per-frame motion features in $\Theta$, we can reconstruct motion tokens in a topology-agnostic manner without requiring explicit graph traversal during generation.

\paragraph{Difference from Prior Spatio-Temporal Transformers.}
Unlike prior motion Transformers that operate on fixed human skeletons, 
the proposed TAT introduces three key innovations:

\begin{itemize}
    \item \textbf{Graph-conditioned spatial attention.} 
    Each joint feature is modulated by its OwO embedding, 
    allowing the spatial attention module to reason over anatomical semantics rather than relying solely on joint indices.

    \item \textbf{Topology-invariant quantization via the virtual joint.}
    Only the virtual joint is quantized, 
    enabling a discrete representation that is independent of the number or ordering of joints.

    \item \textbf{Decoupled structure–motion representation.}
    The motion tokens contain no structural information; 
    the structure is injected only through OwO at encode and decode time. 
    This makes the latent motion code universally applicable to any skeleton.
\end{itemize}

Together, these innovations make TAT the first spatio-temporal tokenizer capable of 
reconstruction, transfer, and generation across arbitrary skeletal topologies.

\begin{table*}[h]
\centering
\caption{Reconstruction results on three datasets. Lower is better.}
\label{tab:main_recon}
\setlength{\tabcolsep}{6pt}
\small
\begin{tabular}{l rrr rrr rrr}
\toprule
\multirow{2}{*}{Method} &
\multicolumn{3}{c}{MPJPE $\downarrow$} &
\multicolumn{3}{c}{MPJPE (no trans.) $\downarrow$} &
\multicolumn{3}{c}{GeoDist $\downarrow$} \\
\cmidrule(lr){2-4}\cmidrule(lr){5-7}\cmidrule(lr){8-10}
& H3D & Obj-XL & Zoo
& H3D & Obj-XL & Zoo
& H3D & Obj-XL & Zoo \\
\midrule
T2M-GPT~\cite{zhang2023generating}          & 0.4203 & 0.2583 & 0.2271 & 0.1376 & 0.2128 & 0.0843 & 6.84$^\circ$  & 28.66$^\circ$ & 18.76$^\circ$ \\
Motion Streamer~\cite{xiao2025motionstreamer}  & 0.1961 & 0.2456 & 0.2261 & 0.0972 & 0.1963 & 0.0842 & 5.37$^\circ$  & 26.17$^\circ$ & 18.73$^\circ$ \\
TM2T~\cite{guo2022tm2t}             & 0.1411 & 0.1918 & 0.1434 & 0.0873 & 0.1565 & 0.0757 & 5.34$^\circ$  & 21.49$^\circ$ & 17.28$^\circ$ \\
\midrule
RVQ-VAE (zero pad) & 0.4729 & 0.2228 & 0.1762 & 0.2688 & 0.1817 & 0.1143 & 27.95$^\circ$ & 22.72$^\circ$ & 18.76$^\circ$ \\
\textbf{NEC w/ VQ}  & 0.3960 & 0.1840 & 0.1657 & 0.1395 & 0.1343 & 0.0828 & 7.78$^\circ$  & 17.40$^\circ$ & 16.19$^\circ$ \\
\textbf{NEC w/ RVQ} & \textbf{0.1084} & \textbf{0.0983} & \textbf{0.1008}
                   & \textbf{0.0588} & \textbf{0.0787} & \textbf{0.0635}
                   & \textbf{3.96$^\circ$} & \textbf{12.12$^\circ$} & \textbf{13.88$^\circ$} \\
\bottomrule
\end{tabular}
\end{table*}

\paragraph{Spatio-temporal Modeling.}
As shown in Figure~\ref{fig:tokenizer}, the encoder contains \( L \) spatio-temporal blocks. Each block consists of:
A \textbf{temporal module}, comprising a 1D convolution (with stride \( s \)) and a ResNet1D block for temporal abstraction; A \textbf{spatial transformer}, which models joint interactions at each timestep using multi-head self-attention.

At each layer, temporal downsampling reduces the sequence length by a factor of \( s \), resulting in an overall downsampling rate \( r = s^L \). The final feature is reshaped and passed through a \(1 \times 1\) convolution to produce quantized features \( \mathbf{Z}_{\text{feat}} \in \mathbb{R}^{T/r \times J \times W} \), from which token quantization is performed. The virtual joint token is extracted separately and used as a global summary feature.

The decoder mirrors the encoder structure in reverse. Given a quantized token sequence and the corresponding joint features, it performs: Concatenation of per-joint features and the global token; 
Spatial transformer operations for joint-wise refinement; Upsampling and reverse-dilated temporal convolutions to restore full temporal resolution.

The virtual joint token is removed before output. The decoder maps the output back to the original motion space \( \mathbb{R}^{T \times J \times 9} \) using a linear projection. This design allows the tokenizer to encode complex spatio-temporal patterns in a structure-aware yet data-efficient manner, and produce discrete tokens suitable for downstream generative modeling. Details can be found in the supplementary materials.

\subsection{Data Augmentation}
To enhance topological generalization, we propose a base pose randomization technique that preserves semantic content while expanding motion diversity. The method involves: (1) selecting a random frame as new rest pose, (2) computing global rotations and positions, (3) deriving new rest pose offsets, and (4) recalculating local rotations using key equations (see Appendix~\ref{sec:data_formul}). Our data augmentation generates physically plausible animations with semantically consistent motions, encouraging the model to prioritize semantic meaning and physical correctness.
\section{Experiments}
\label{sec:experiments}

\begin{table*}[t]
\centering
\caption{
Ablation study on pretraining and loss configurations.
We report results on three tasks across three datasets.
Lower is better for MPJPE, higher is better for R-Precision@1.
}
\label{tab:ablation_loss}
\setlength{\tabcolsep}{5.4pt}
\small
\begin{tabular}{c cccc ccc ccc ccc}
\toprule
\multirow{2}{*}{Pretrain} &
\multicolumn{4}{c}{Loss Setting} &
\multicolumn{3}{c}{Pose VAE (MPJPE)} &
\multicolumn{3}{c}{Motion VQ-VAE (MPJPE)} &
\multicolumn{3}{c}{Evaluator (R-Precision@1)} \\
\cmidrule(lr){2-5}
\cmidrule(lr){6-8}
\cmidrule(lr){9-11}
\cmidrule(lr){12-14}
& Offset & LCA & Dist. & Con.
& H3D & Obj-XL & Zoo
& H3D & Obj-XL & Zoo
& H3D & Obj-XL & Zoo \\
\midrule
No  & 0 & 0 & 0 & 0
& 0.0940 & 0.0618 & 0.0599
& 0.1458 & 0.1034 & \textbf{0.0995}
& 0.5355 & 0.1655 & 0.2756 \\

Yes & 1 & 0 & 0 & 0
& 0.0443 & 0.0442 & 0.0485
& 0.1151 & 0.1028 & 0.1122
& 0.4052 & 0.1310 & 0.2677 \\

Yes & 0 & 1 & 0 & 0
& 0.0530 & 0.0514 & 0.0585
& 0.1300 & 0.1051 & 0.1529
& 0.3913 & 0.1448 & 0.2913 \\

Yes & 0 & 0 & 1 & 0
& 0.0623 & 0.0412 & 0.0466
& 0.1138 & 0.1024 & 0.1146
& 0.3731 & 0.1793 & \textbf{0.3150} \\

Yes & 0 & 0 & 0 & 1
& 0.0399 & 0.0502 & \textbf{0.0433}
& 0.1126 & 0.1050 & 0.1006
& 0.5215 & 0.1172 & 0.2913 \\

Yes & 0 & 1 & 1 & 1
& 0.0822 & \textbf{0.0409} & 0.0485
& 0.1084 & \textbf{0.0983} & 0.1008
& \textbf{0.5713} & \textbf{0.2207} & 0.2992 \\

Yes & 1 & 0 & 1 & 1
& 0.0627 & 0.0424 & 0.0475
& \textbf{0.1073} & 0.1017 & 0.1359
& 0.5604 & 0.1586 & 0.3071 \\

Yes & 1 & 1 & 0 & 1
& 0.0780 & 0.0410 & 0.0460
& 0.1343 & 0.1098 & 0.1058
& 0.5580 & 0.1586 & 0.2283 \\

Yes & 1 & 1 & 1 & 0
& \textbf{0.0363} & 0.0411 & 0.0521
& 0.1567 & 0.1074 & 0.1003
& 0.5434 & 0.1793 & 0.2835 \\

Yes & 1 & 1 & 1 & 1
& 0.1009 & 0.0516 & 0.0661
& 0.1147 & 0.1088 & 0.1144
& 0.4976 & 0.1793 & 0.2992 \\
\bottomrule
\end{tabular}
\end{table*}

%主表换了这块就不太对了，根据主表重写一下,以及添加了motion streamer的引用
\paragraph{Baselines.}
We compare against three types of baselines.

(1) \emph{Human-centric text-to-motion models}, including T2M-GPT\cite{zhang2023generating}, Motion Streamer\cite{xiao2025motionstreamer}, and TM2T\cite{guo2022tm2t}.
Since these methods are originally designed for fixed human skeletons, we minimally adapt their input and output interfaces to support BVH motions with varying joint sets, while keeping their core model architectures and training objectives unchanged.
Specifically, all motions are zero-padded to the required joint layout of each model.

(2) \emph{Padding-based tokenizers}.
We evaluate RVQ-VAE trained on zero-padded BVH sequences under a single canonical skeleton.
We do not include a separate zero-padded VQ baseline, as it is functionally equivalent to T2M-GPT.

(3) \emph{Our variants}, including NEC w/ VQ and NEC w/ RVQ, which share the same architecture and differ only in the quantization scheme.
All baselines are trained and evaluated under the same unified protocol unless otherwise specified.

% \paragraph{Baselines.}
% We compare against canonical \textbf{VQVAE}~\cite{van2017neural} and \textbf{RVQVAE}~\cite{guo2023momask} tokenizers trained on padded/masked sequences to a max-joint layout (single canonical skeleton).
% Our method \textbf{NEC} combines the \textbf{OwO} (ontology-aware skeletal graph encoder) with the \textbf{TAT} tokenizer to produce topology-agnostic motion tokens on arbitrary BVH skeletons.

\paragraph{Implementation details.}
Unless stated, OwO uses 8 graph attention blocks with 512 latent dimension while TAT uses $3$ spatio-temporal blocks. Each spatio-temporal block contains 2 convolution layers and 2 spatial transformer encoder layers. The used quantizer is a 6-layer  codebooks (size 1024).
We train with AdamW, cosine LR, gradient clipping, and random rest-pose augmentation (Sec.~\ref{sec:methods}). The tokenizer is trained with 32 Ascend 910. The whole training process contains around 24k iterations. Initial learning rate is 2e-4 and is decreased to 2e-5 during the last 4k iterations.

\paragraph{Metrics.}
We report four metrics throughout:
\textbf{MPJPE} $\downarrow$ (root-aligned, joint-set-agnostic),
\textbf{FID} $\downarrow$ (computed on the same retrieval backbone as prior work),
\textbf{GeoDist} $\downarrow$ (mean geodesic distance of joint rotations),
and \textbf{R-Precision@$\{1,2,3\}$} $\uparrow$ for text--motion retrieval.

\paragraph{Other applications of Graph Embedder.}
For understanding BVH motion, accurately extracting features for each joint is crucial. 
Beyond the tokenizer, we also experimented with two contrastive learning models---PoseVAE and Text–Motion Evaluator---to validate the effectiveness of our proposed Graph Embedder. 
Both models share a similar overall structure with the tokenizer, but with key differences:  
PoseVAE operates on single frames and employs a VAE for modeling, 
whereas the Evaluator applies a transformer along the temporal dimension to extract a unified feature from the motion sequence, 
which is then used to compute similarity with textual representations.

\subsection{Main Results: Reconstruction}
\label{sec:main_results}

% 把那个结构消融放进去了，RVQ的就不放表了吧？
\begin{table*}[t]
\centering
\caption{
OwO structural ablation on reconstruction accuracy.
Lower is better.
}
\label{tab:owo_struct_ablation}
\setlength{\tabcolsep}{6pt}
\small
\begin{tabular}{l rrr rrr rrr}
\toprule
Method
& \multicolumn{3}{c}{MPJPE $\downarrow$}
& \multicolumn{3}{c}{MPJPE (no trans.) $\downarrow$}
& \multicolumn{3}{c}{GeoDist $\downarrow$} \\
\cmidrule(lr){2-4}\cmidrule(lr){5-7}\cmidrule(lr){8-10}
& H3D & Obj-XL & Zoo
& H3D & Obj-XL & Zoo
& H3D & Obj-XL & Zoo \\
\midrule
Learnable Query
& 0.1475 & 0.1585 & 0.1149
& 0.0751 & 0.1295 & 0.0737
& 5.04$^\circ$ & 19.44$^\circ$ & 16.37$^\circ$ \\

+ Joint Name
& 0.1759 & 0.1434 & 0.1324
& 0.0755 & 0.1005 & 0.0645
& 4.93$^\circ$ & 15.80$^\circ$ & 15.15$^\circ$ \\

\textbf{Full OwO (ours)}
& \textbf{0.1084} & \textbf{0.0983} & \textbf{0.1008}
& \textbf{0.0588} & \textbf{0.0787} & \textbf{0.0635}
& \textbf{3.96$^\circ$} & \textbf{12.12$^\circ$} & \textbf{13.88$^\circ$} \\
\bottomrule
\end{tabular}
\end{table*}

\subsection{Ablations}
\label{sec:ablations}

\paragraph{Summary.}
Table~\ref{tab:main_recon} summarizes reconstruction performance across three heterogeneous datasets: HumanML3D~\cite{guo2022generating}, Objaverse-XL~\cite{objaverseXL}, and Truebones Zoo~\cite{truebones_mocap}.
For T2M-GPT, Motion Streamer, TM2T, and the traditional RVQVAE baseline, we pad all skeletons to the fixed joint layout required by these models, ensuring compatibility with their human-centric architectures.

Across all three datasets, a consistent trend emerges.
Human-centric models (T2M-GPT~\cite{zhang2023generating}, Motion Streamer~\cite{xiao2025motionstreamer}, TM2T~\cite{guo2022tm2t}) perform reasonably on HumanML3D but degrade significantly on Objaverse-XL and Zoo, where joint counts, limb structures, and bone-length statistics differ widely. This shows that models assuming fixed human topology cannot generalize to heterogeneous skeletons.
Padding-based RVQ-VAE baselines perform even worse, with high MPJPE(no-trans) and GeoDist indicating that zero-padding introduces structural ambiguity and prevents accurate rotational reconstruction.

By contrast, NEC overcomes these limitations through an ontology-aware structural prior (OwO) and topology-invariant quantization (TAT). Even NEC w/ VQ surpasses all padding-based VQ baselines, showing the benefits of structure-aware encoding.
The full NEC w/ RVQ achieves the strongest results across all datasets and metrics, reducing MPJPE on HumanML3D by nearly 2× compared to TM2T, and delivering large improvements on Obj-XL and Zoo. Its consistently lowest GeoDist further confirms superior rotational accuracy and cross-topology fidelity.

Overall, the results demonstrate that topology-aware conditioning combined with residual quantization is crucial for precise and structurally consistent motion reconstruction across diverse skeletal morphologies.

\paragraph{OwO pretraining: reconstruction + retrieval (Table~2).}
We ablate the effect of \textbf{OwO} pretraining by toggling its auxiliary objectives: Adjacent Offset regression (Offset), LCA/topology prediction (LCA),  Distance regression on any pairs (Dist.), and name-text contrastive learning (Con.). Here Offset is a special case of Dist.

\begin{figure*}[t]
\centering
\includegraphics[width=\textwidth]{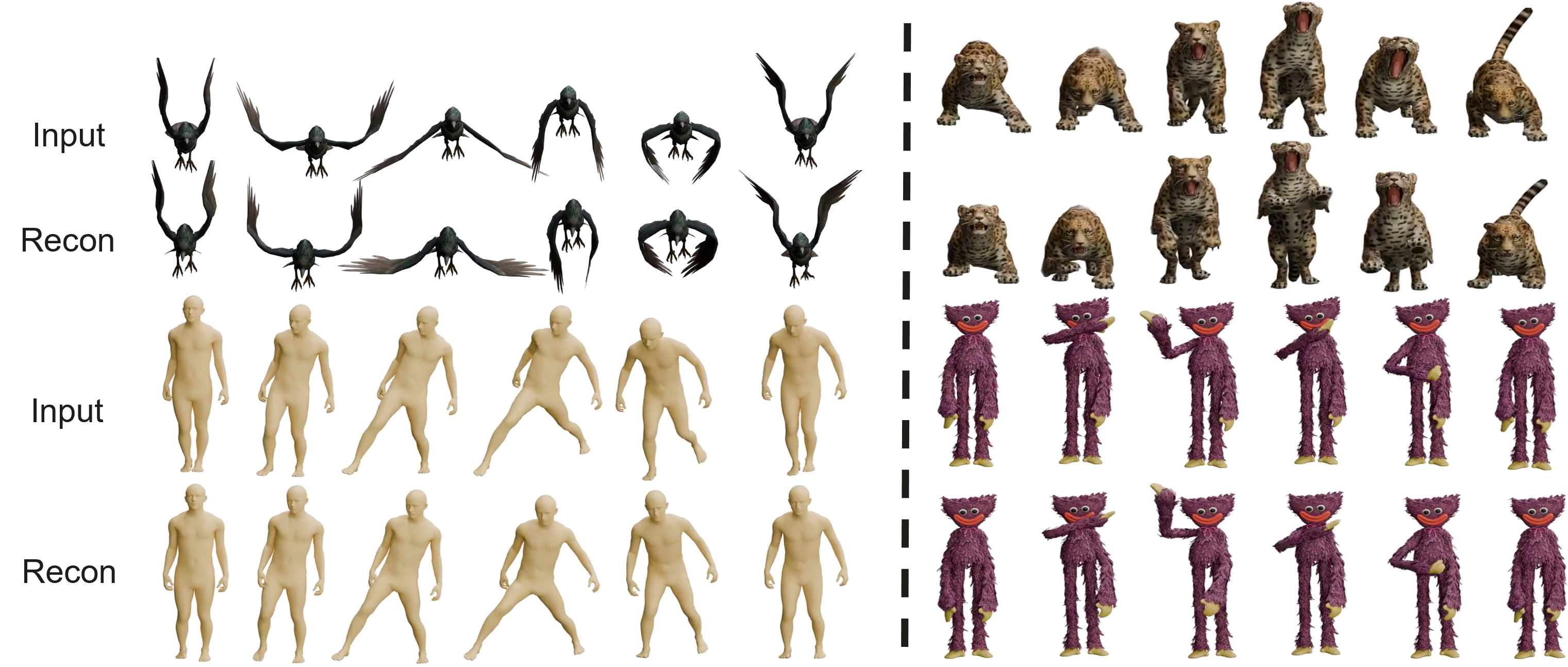}
\caption{Qualitative reconstruction results comparing NEC with ground truth on Objaverse-XL and Truebones.}
\label{fig:qual_recon}
\end{figure*}
%缩一点点
\noindent\textbf{Findings.}
(1) Overall, combining the three loss functions---\emph{LCA}, \emph{Dist.}, and \emph{Con.}---achieves the best performance across all tasks.  
(2) The \emph{Offset} task, which predicts offsets between neighboring joints, is relatively simple and provides limited benefit for representation learning; in fact, it may even have a negative effect. 
(3) The \emph{LCA} task substantially improves performance on the Zoo dataset, likely due to its highly diverse skeletal structures, indicating that learning topology enhances cross-skeleton generalization.

% \noindent\textbf{Findings.}
% (1) Combining all three losses (\emph{LCA}, \emph{Dist.}, \emph{Con.}) yields the best overall performance. 
% (2) The \emph{Offset} objective is simple and offers little benefit, and can even hurt representation learning. 
% (3) \emph{LCA} notably improves performance on Zoo, likely because its skeletons are highly diverse, suggesting topology learning enhances cross-skeleton generalization.

% \paragraph{OwO structural ablation.} To further isolate the contribution of structural priors in OwO, we compare three variants: (i) a single learnable query without joint semantics or graph attention, (ii) a variant augmented with joint-name semantics only, and (iii) the full OwO with ontology-aware graph attention. As shown in Table~\ref{tab:owo_struct_ablation}, removing graph-structure reasoning leads to a consistent degradation across all datasets and metrics. Joint-name semantics alone provide limited improvement, 
% while the full OwO achieves the best performance by a clear margin, demonstrating that both semantic cues and topological reasoning are essential for topology-agnostic motion representation.

\paragraph{OwO structural ablation.}
We compare: (i) a single learnable query without joint semantics or graph attention, (ii) a variant augmented with joint-name semantics only, and (iii) the full OwO with ontology-aware graph attention.
Table~\ref{tab:owo_struct_ablation} shows that removing graph reasoning consistently degrades performance across datasets and metrics. Joint-name semantics alone offer limited gains, whereas full OwO performs best, indicating that both factors are crucial for topology-agnostic motion representation.

\subsection{Qualitative Results}
\label{sec:qual}

% \paragraph{Reconstruction vs. baselines.}
% Fig.~\ref{fig:qual_recon} shows side-by-side BVH pose sequences for NEC, VQVAE, and RVQVAE on HumanML3D/Objaverse-XL/Truebones.
% NEC better preserves limb phasing and contact timing, especially on non-human skeletons.

Fig.~\ref{fig:qual_recon} shows reconstruction results from NEC. The species including a standard SMPL body mesh, a humanoid character, a quadruped mammal and a bird, which covers a large range of species used in our benchmark. Our tokenizer can basically recover the motion tendency with sufficient motion details. Admittedly, certain fine-grained pose details cannot yet be reproduced with high precision, leaving room for further improvement.

% 如果弱化，这边应该也省了
% \paragraph{Motion transfer through different species.}
% We demonstrate \textbf{motion transfer} (source tokens + target OwO) across arbitrary topologies (Fig.~\ref{fig:qual_token_ops}).

% We demonstrate \textbf{motion transfer} (source tokens + target OwO) across arbitrary topologies.
% \emph{Crucially, this requires no task-specific fine-tuning or auxiliary heads}: because NEC produces \emph{topology-agnostic} tokens and the decoder is \emph{OwO-conditioned}, a single trained model supports \textbf{zero-shot transfer} by decoding source token sequences under different target OwOs (morphology swap).
% This enables direct motion transfer between species with diverse skeletons while maintaining temporal coherence.

% \paragraph{Plug-and-play generation.}
% We plug \textbf{NEC} tokens into a token-based generator (e.g., MoMask) for \textbf{any-skeleton} text-to-motion), using only target OwO to condition morphology.
% More generally, \emph{any} token-based generator can be trained or fine-tuned on NEC tokens without architectural changes, treating NEC codes as the discrete vocabulary and injecting morphology at decode time. Please refer to the demo video for the generation results.

Benefiting from the design that uses the skeleton as an additional conditioning signal, our method can be applied to motion transfer. Specifically, we feed motion sequence from species A into the encoder to obtain a set of discrete tokens, and then pass these tokens to the decoder together with the conditioning signal of species B, thereby enabling cross-species motion transfer. Another use case is to combine our approach with classic discrete-token-based motion generation methods, enabling text-driven motion generation for any species. Please refer to the supplementary material for detailed experiments on both parts.

% \begin{figure*}[t]
% \centering
% \includegraphics[width=\textwidth]{images/transfer_add_line}
% \caption{Qualitative motion transfer results across different skeletons and object categories.}
% \label{fig:qual_token_ops}
% \end{figure*}
% \input{sections/6_applications}
\section{Conclusions}
\label{sec:conclusions}

% We presented a unified framework for motion representation learning with three key contributions.  
% First, the \emph{Ontology-aware Skeletal Graph Encoder} leverages graph embeddings and self-supervised losses to pre-train on rigged 3D data, yielding meta-representations that capture topology, posture, and semantics.  
% Second, the \emph{Topology-Agnostic Tokenizer} converts arbitrary BVH skeletons into compact, skeleton-agnostic tokens, enabling diverse downstream applications.  
% Third, we established a large-scale \emph{BVH-centric benchmark} (47{,}807 sequences) covering heterogeneous species and skeletons, supporting standardized evaluation for reconstruction, retrieval, and distributional quality. These contributions together provide a foundation for topology-generalized motion modeling and evaluation in BVH format.

We presented a unified framework for topology-generalized motion representation learning.
Our approach integrates an \emph{Ontology-aware Skeletal Graph Encoder} for learning structural and semantic priors from rigged 3D data, a \emph{Topology-Agnostic Tokenizer} that converts arbitrary BVH skeletons into compact discrete tokens, and a large-scale \emph{BVH-centric benchmark} (47{,}807 sequences) covering heterogeneous species and skeletons.
Together, these components establish a practical foundation for reconstruction, retrieval, and generation across diverse skeletal topologies in BVH format.

\section*{Impact Statement}
This paper presents work whose goal is to advance the field of machine learning. There are many potential societal consequences of our work, none of which we feel must be specifically highlighted here.

% ================= References =================
% \clearpage
\bibliography{icml2026}
\bibliographystyle{icml2026}

% ================= Appendix =================
\newpage
\appendix
\onecolumn
\appendix

{\LARGE\sc {Appendix for NECromancer}\par}

\section*{LLM Usage}
A large language model was used as a writing assistant to improve the clarity and readability of the manuscript. 
In addition, a vision-language model was employed to assist in data preprocessing and filtering. 
All final research decisions, data selection criteria, and methodological contributions were made and verified by the authors.

\section{Details of BVH Representation}

\begin{figure*}[t]
  \centering
  \includegraphics[width=\textwidth]{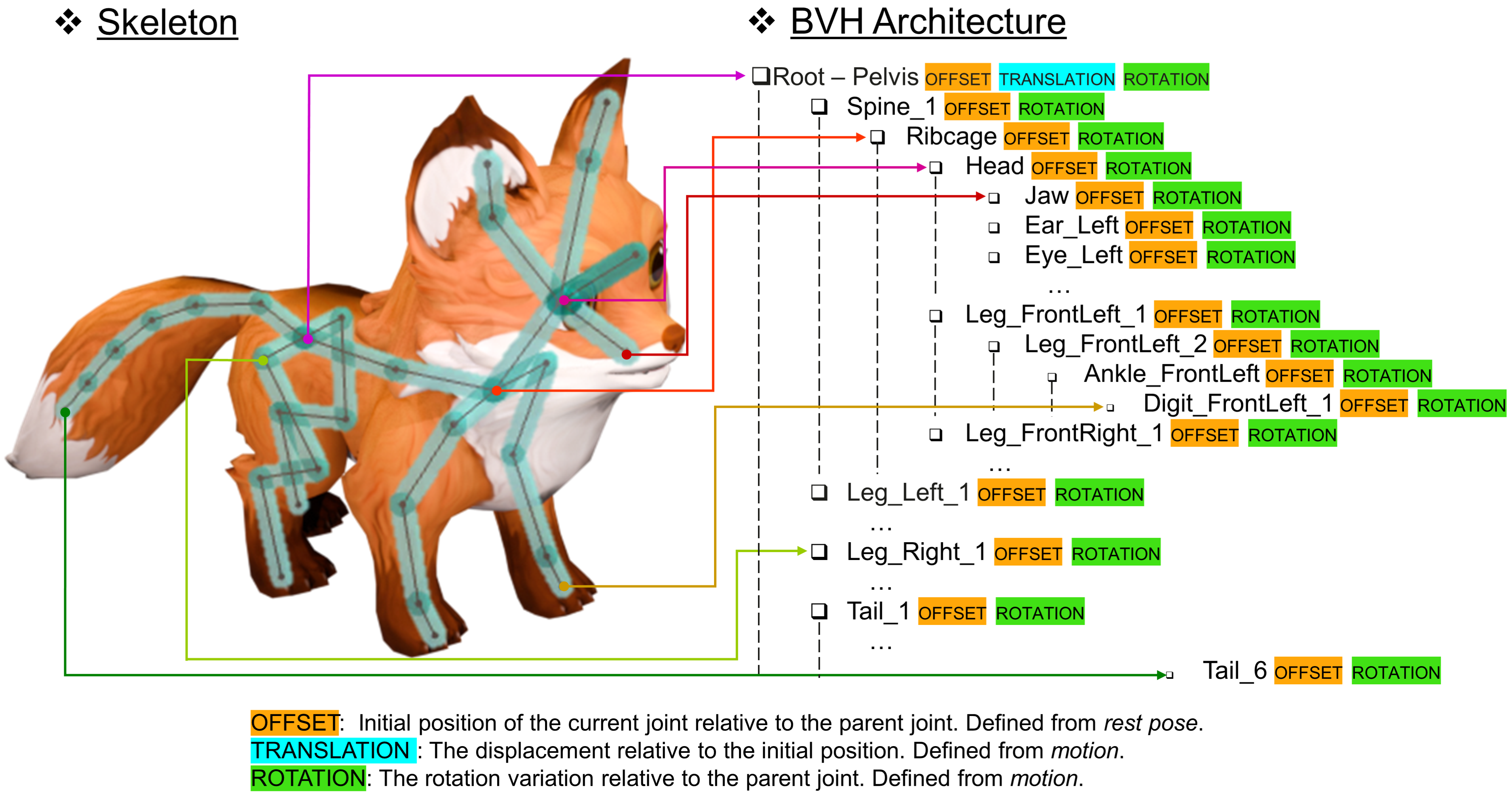}
  \caption{
  \textbf{Overview of the BVH motion format.}
  BVH encodes a skeleton as a joint hierarchy with fixed rest-pose offsets (\texttt{OFFSET}), and represents motion as per-frame channels (\texttt{TRANSLATION} for the root and \texttt{ROTATION} for all joints) applied in hierarchical order.
  }

  \label{fig:bvh}
\end{figure*}

To enable motion generation across diverse skeletal structures, we adopt BVH (Biovision Hierarchy) as the unified representation format for all skeleton-based animations in our framework. BVH is a widely used hierarchical motion representation format that encodes both the skeletal topology (rest pose) and temporal motion trajectories in a compact and interpretable structure, as shown in Fig.~\ref{fig:bvh}. Its compatibility with commercial animation software and motion datasets makes it an ideal choice for bridging learning-based motion generation with practical deployment.

In this section, we provide a detailed explanation of the two critical components of the BVH format: the rest pose, which defines the static skeleton configuration, and the motion representation, which captures dynamic joint transformations over time.

\subsection{Rest Pose}

The rest pose in a BVH file encodes the skeleton hierarchy, which defines the parent-child relationship between joints, as well as the offset vectors between them. Each joint is represented by its name, position relative to its parent, and its degrees of freedom (DOF), such as rotation along the X, Y, and Z axes. The hierarchy is generally rooted at the Hips joint or an equivalent base (e.g., the Pelvis joint), and traverses downward through limbs and extremities.

Importantly, the rest pose serves as the topological and spatial reference for interpreting all subsequent motion data. In our framework, we treat the rest pose as a graph, where nodes correspond to joints and edges reflect the skeletal hierarchy. This representation allows us to extract structural embeddings for each skeleton, facilitating the generalization of motion generation across arbitrary topologies.
To ensure consistency across heterogeneous sources, we standardize joint naming under a consistent convention.
% normalize all skeletons to a standard coordinate system and unify joint naming with convention. 

\subsection{Motion Representation}

The motion section of a BVH file records the transformations per frame applied to each joint over time. Each frame typically contains a set of scalar values corresponding to the DOF specified in the rest pose, usually in the form of Euler angles and root joint translations. These values are organized in a flat sequence for each time step, but semantically correspond to articulated motion governed by the skeletal hierarchy.

To handle skeletons of varying topology and DOF, we dynamically build the hierarchy graph for each BVH file based on its rest pose. This ensures that all motion sequences, regardless of their original structure, can be consistently represented and processed under a unified format.

\section{Details of Data Preprocessing}
To ensure a consistent and physically meaningful skeletal representation across the heterogeneous HumanML3D, Truebones Zoo and Objaverse-XL datasets, we apply a unified and carefully designed preprocessing pipeline.

\textbf{(1) Correction of invalid or awkward root joint definitions.}
Several skeletons in both datasets adopt unconventional root placements (e.g., a dummy root joint set outside of the subject body, or root incorrectly set to a shoulder or spine joint). 
Such configurations violate common skeletal conventions and introduce instability for models relying on hierarchical transformations. 
We identify these problematic cases and relocate the root to an anatomically meaningful position (typically the pelvis), reconstructing the hierarchy so that all parent–child relationships follow consistent human and animal body kinematics.

\textbf{(2) Reassignment of global translations to the root joint.}
In the raw data, some sequences embed joint-level translations on intermediate nodes, leading to frame-dependent bone-length drift and violating rigid-body kinematic constraints. 
To restore physical consistency, we transfer all global motion to the root joint and remove per-joint translations from all other nodes. 
This guarantees that bone lengths remain constant across frames and that only the root carries global displacement, while the rest of the skeleton expresses purely rotational motion.

\textbf{(3) Removal of sequences with abnormal bone lengths.}
A subset of sequences contains extremely elongated or corrupted bones due to tracking errors or incorrect parameter export. 
Since such cases break kinematic plausibility and cannot be reliably corrected without strong assumptions, we exclude sequences whose bone lengths exceed a dataset-dependent threshold relative to their normalized diameter, effectively filtering out physically invalid examples.

\textbf{(4) Semantic standardization of joint names using a vision-language model.}
Joint names provided in these datasets are highly inconsistent and often ambiguous (e.g., ``joint1'', ``bone\_02'', ``arm.L'', or dataset-specific conventions).
To achieve semantic uniformity, we employ a vision-language model (VLM) to map \emph{raw joint names} to a canonical semantic vocabulary (including non-human anatomical terms when applicable).
This step not only harmonizes naming across datasets, but also ensures consistent semantic interpretation for downstream tasks such as retargeting, pose comparison, and learned conditioning.

\textbf{(5) Scale normalization based on skeletal diameter.}
The raw datasets contain skeletons defined at widely different scales, resulting in inconsistent geometry and bone-length statistics. 
For each sequence, we compute the skeleton diameter, defined as the length of the longest kinematic chain from the root to any leaf joint. 
This measure offers a robust scale descriptor that is less sensitive to local bone-length noise. 
We uniformly scale the entire sequence using this diameter, ensuring that all skeletons across datasets share comparable global scale while preserving their relative proportions.

After applying the above stages, every sequence is represented by a single, coherent skeletal tree with fixed bone lengths, a unified joint naming scheme, and global motion encoded exclusively at the root. 
All redundant, duplicated, or semantically meaningless joints are removed. 
This preprocessing ensures topological and geometric consistency across datasets and enables fair and stable learning for all subsequent modules.

% \begin{table*}[t]
% \centering
% \caption{Generation results on three datasets. }
% \label{tab:main_gene}
% \setlength{\tabcolsep}{6pt}
% \small
% \begin{tabular}{l rrr rrr rrr}
% \toprule
% \multirow{2}{*}{Method} &
% \multicolumn{3}{c}{R-Precision@1 $\uparrow$} &
% \multicolumn{3}{c}{R-Precision@3 $\uparrow$} &
% \multicolumn{3}{c}{FID $\downarrow$} \\
% \cmidrule(lr){2-4}\cmidrule(lr){5-7}\cmidrule(lr){8-10}
% & H3D & Obj-XL & Zoo
% & H3D & Obj-XL & Zoo
% & H3D & Obj-XL & Zoo \\
% \midrule
% T2M-GPT~\cite{zhang2023generating}          & 0.0000 & 0.0000 & 0.0000 & 0.0000 & 0.0000 & 0.0000 & 0.0000  & 0.0000 & 0.0000 \\
% Motion Streamer~\cite{xiao2025motionstreamer}  & 0.0000 & 0.0000 & 0.0000 & 0.0000 & 0.0000 & 0.0000 & 0.0000  & 0.0000 & 0.0000 \\
% TM2T~\cite{guo2022tm2t}             & 0.0391 & 0.1513 & 0.0000 &  0.1035 & 0.3601 &  0.0000  & 1.2510 & 1.031 & 0.0000 \\
% \midrule
% RVQ-VAE (zero pad) & 0.0000 & 0.0000 & 0.0000 & 0.0000 & 0.0000 & 0.0000 & 0.0000 & 0.0000 & 0.0000 \\
% \textbf{NEC w/ VQ}  & 0.0000 & 0.0000 & 0.0000 & 0.0000 & 0.0000 & 0.0000 & 0.0000  & 0.0000 & 0.0000 \\
% \textbf{NEC w/ RVQ} & \textbf{0.0000} & \textbf{0.0000} & \textbf{0.0000}
% & \textbf{0.0000} & \textbf{0.0000} & \textbf{0.0000}
% & \textbf{0.0000} & \textbf{0.0000} & \textbf{0.0000} \\
% \bottomrule
% \end{tabular}
% \end{table*}

\section{Implementation Details}
\label{sec:app_imp}

\subsection{Data Augmentation: Mathematical Formulations}
\label{sec:data_formul}
The proposed data augmentation method follows these mathematical formulations:

1. Global rotation calculation:
\begin{equation}
\mathbf{r}_{t_0, i}^{\text{global}} = \mathbf{r}_{t_0, \text{parent}[i]}^{\text{global}} \cdot \mathbf{r}_{t_0, i}^{\text{local}} \tag{1}
\end{equation}

2. Global position calculation:
\begin{equation}
\mathbf{p}_{t_0, i}^{\text{global}} = \mathbf{p}_{t_0, \text{parent}[i]}^{\text{global}} + \mathbf{r}_{t_0, \text{parent}[i]}^{\text{global}} \cdot \mathbf{o}_{i} \tag{2}
\end{equation}

3. Rest pose offset derivation:
\begin{equation}
\mathbf{\bar{o}}_{i} = \mathbf{p}_{t, i}^{\text{global}} - \mathbf{p}_{t, \text{parent}[i]}^{\text{global}} \tag{3}
\end{equation}

4. Local rotation re-calculation:
\begin{equation}
\mathbf{\bar{r}}_{t, i}^{\text{local}} = (\mathbf{\bar{r}}_{t, \text{parent}[i]}^{\text{global}})^{-1} \cdot \mathbf{r}_{t, i}^{\text{global}} \cdot (\mathbf{r}_{t_0, i}^{\text{global}})^{-1} \tag{4}
\end{equation}

With such data augmentation process, we synthesize a large amount of animations with reasonable base poses, which is more physically plausible compared to random rotation jittering. By training on many animations that are different in numeric representations but the same on semantic meanings, the model is strongly encouraged to focus more on the semantic meaning and physical correctness of the animations.

\subsection{Graph Embedder}

\subsubsection{Graph Initialization}

Given a rest pose skeleton \( \mathcal{S} = (\mathcal{J}, \mathcal{E}) \), where \( \mathcal{J} \) is the set of joints and \( \mathcal{E} \subseteq \mathcal{J} \times \mathcal{J} \) represents parent-child edges, we construct a directed graph where each joint is treated as a node.

To enable graph-based message passing, we initialize three types of features: node features, edge features, and a special global node.

\subsubsection{Node Feature Initialization}

Each joint \( j \in \mathcal{J} \) is annotated with a semantic name. We use a pretrained CLIP text encoder to extract a 512-dimensional embedding \( \mathbf{e}_j^{\text{text}} \in \mathbb{R}^{512} \), which is projected to the graph dimension:
\begin{equation}
    \mathbf{h}_j^{(0)} = \text{FC}\left( \mathbf{e}_j^{\text{text}} \right) \in \mathbb{R}^{d}.
\end{equation}

\subsubsection{Global Node Initialization}

We introduce a special node with name "global" and extract its CLIP embedding in the same manner:
\begin{equation}
    \mathbf{h}_{\text{global}}^{(0)} = \text{FC} \left( \mathbf{e}_{\text{global}}^{\text{text}} \right) \in \mathbb{R}^{d}.
\end{equation}

The global node is connected bidirectionally to all other joints:
\[
\mathcal{E}_{\text{global}} = \left\{ (\text{global} \rightarrow j),\ (j \rightarrow \text{global}) \mid j \in \mathcal{J} \right\}.
\]

These edges are treated identically to regular edges during attention, allowing the global node to gather holistic context from the graph while also distributing high-level signals.

\subsubsection{Edge Feature Initialization}

For each directed edge \( (i \rightarrow j) \in \mathcal{E} \cup \mathcal{E}_{\text{global}} \), we define a bidirectional edge feature.

\textbf{Forward edge:}
\begin{equation}
    \mathbf{e}^{(0)}_{i \rightarrow j} = \text{FC} \left(
        \left[
            \mathbf{h}_i^{(0)} \, \Vert \,
            \mathbf{h}_j^{(0)} \, \Vert \,
            \boldsymbol{\phi}(\mathbf{o}_{i \rightarrow j})
        \right]
    \right),
\end{equation}

\textbf{Backward edge:}
\begin{equation}
    \mathbf{e}^{(0)}_{j \rightarrow i} = \text{FC} \left(
        \left[
            \mathbf{h}_j^{(0)} \, \Vert \,
            \mathbf{h}_i^{(0)} \, \Vert \,
            \boldsymbol{\phi}(-\mathbf{o}_{i \rightarrow j})
        \right]
    \right),
\end{equation}
where \( \mathbf{o}_{i \rightarrow j} \in \mathbb{R}^3 \) is the offset from joint \( i \) to joint \( j \), \(\Vert \) denotes the concatenation operation, and \( \boldsymbol{\phi}(\cdot) \in \mathbb{R}^{3d} \) is a sinusoidal embedding per dimension.

For global-to-joint edges, we set \( \mathbf{o}_{i \rightarrow j} = \mathbf{0} \), but retain the learned joint names for each endpoint.

\begin{table*}[t]
\centering
\caption{
Ablation study of pretraining and loss configurations on \textbf{motion--text transfer retrieval} (top-$K$ R-Precision).
}

\label{tab:ablation_transfer_loss}
\setlength{\tabcolsep}{5.4pt}
\small
\begin{tabular}{c cccc ccc ccc ccc}
\toprule
\multirow{2}{*}{Pretrain} &
\multicolumn{4}{c}{Loss Setting} &
\multicolumn{3}{c}{R-Precision@1 $\uparrow$} &
\multicolumn{3}{c}{R-Precision@2 $\uparrow$} &
\multicolumn{3}{c}{R-Precision@3 $\uparrow$} \\
\cmidrule(lr){2-5}
\cmidrule(lr){6-8}
\cmidrule(lr){9-11}
\cmidrule(lr){12-14}
& Offset & LCA & Dist. & Con.
& H3D & Obj-XL & Zoo
& H3D & Obj-XL & Zoo
& H3D & Obj-XL & Zoo \\
\midrule
No  & 0 & 0 & 0 & 0
& 0.1973 & 0.1211 & 0.0840
& 0.3125 & 0.2012 & 0.1484
& 0.3906 & 0.2617 & 0.2070 \\

Yes & 1 & 0 & 0 & 0
& 0.4688 & 0.1563 & 0.0801
& \textbf{0.6699} & 0.2441 & 0.1660
& 0.7500 & 0.3184 & \textbf{0.2129} \\

Yes & 0 & 1 & 0 & 0
& 0.2637 & 0.1406 & 0.0781
& 0.4063 & 0.2324 & 0.1406
& 0.5020 & 0.3203 & 0.1934 \\

Yes & 0 & 0 & 1 & 0
& 0.0645 & 0.1406 & 0.0586
& 0.1230 & 0.2324 & 0.0938
& 0.1758 & 0.3203 & 0.1367 \\

Yes & 0 & 0 & 0 & 1
& 0.0879 & 0.0762 & 0723
& 0.1523 & 0.1426 & 0.1348
& 0.2207 & 0.2070 & 0.1797 \\

Yes & 0 & 1 & 1 & 1
& 0.4727 & 0.1523 & \textbf{0.0898}
& 0.6484 & 0.2539 & \textbf{0.1504}
& \textbf{0.7520} & 0.3320 & 0.1934 \\

Yes & 1 & 0 & 1 & 1
& \textbf{0.4902} & 0.1465 & 0.0684
& 0.6387 & 0.2480 & 0.1191
& 0.7246 & 0.3281 & 0.1797 \\

Yes & 1 & 1 & 0 & 1
& 0.2363 & 0.1172 & 0.0801
& 0.3633 & 0.2051 & 0.1270
& 0.4355 & 0.2852 & 0.1660 \\

Yes & 1 & 1 & 1 & 0
& 0.2266 & 0.1406 & 0.0820
& 0.3809 & 0.1934 & 0.1465
& 0.4746 & 0.2656 & 0.1992 \\

Yes & 1 & 1 & 1 & 1
& 0.4824 & \textbf{0.2012} & 0.0840
& 0.6348 & \textbf{0.3242} & 0.1406
& 0.7246 & \textbf{0.3945} & 0.2070 \\
\bottomrule
\end{tabular}
\end{table*}

\subsubsection{Graph Encoder Block}

The graph is processed by \( L \) stacked Graph Encoder Blocks, each consisting of a Graph Attention Layer and a Feed-Forward Network (FFN).

\subsubsection{Graph Attention Layer}

Let \( \mathbf{h}_j^{(l)} \) be the feature of node \( j \) at layer \( l \). We compute:
\begin{equation}
    \tilde{\mathbf{h}}_j^{(l+1)} = \sum_{k \in \mathcal{N}(j)} \alpha_{k \rightarrow j}^{(l)} \cdot \mathbf{v}_{k \rightarrow j}^{(l)},
\end{equation}
where \( \alpha_{k \rightarrow j}^{(l)} \in [0,1] \) is the attention weight, and \( \mathbf{v}_{k \rightarrow j}^{(l)} \in \mathbb{R}^d \) is the edge-conditioned message:
\begin{equation}
    \alpha_{k \rightarrow j}^{(l)} = \frac{\exp \left( \text{LeakyReLU} \left( \mathbf{a}^\top 
    \left[ 
        \mathbf{W}_{\mathrm{node}}  \mathbf{h}_j^{(l)} \, \Vert \,
        \mathbf{W}_{\mathrm{node}}  \mathbf{h}_k^{(l)} \, \Vert \,
        \mathbf{e}_{k \rightarrow j}^{(l)}
    \right] \right) \right)}{
        \sum_{k' \in \mathcal{N}(j)} \exp \left( \text{LeakyReLU} \left( \mathbf{a}^\top 
        \left[ 
            \mathbf{W}_{\mathrm{node}}  \mathbf{h}_j^{(l)} \, \Vert \,
            \mathbf{W}_{\mathrm{node}}  \mathbf{h}_{k'}^{(l)} \, \Vert \,
            \mathbf{e}_{k' \rightarrow j}^{(l)}
        \right] \right) \right)},
\end{equation}
\begin{equation}
    \mathbf{v}_{k \rightarrow j}^{(l)} = \mathbf{W}_v \left[ \mathbf{h}_k^{(l)} \, \Vert \, \mathbf{e}_{k \rightarrow j}^{(l)} \right].
\end{equation}
\begin{equation}
    \tilde{\mathbf{e}}_{k \rightarrow j}^{(l
    +1)} = \mathbf{e}_{k \rightarrow j}^{(l
    )} + \mathbf{W}_{\mathrm{src}}  \mathbf{W}_{\mathrm{node}}  \mathbf{h}_j^{(l)} +
            \mathbf{W}_{\mathrm{tgt}}  \mathbf{W}_{\mathrm{node}}  \mathbf{h}_{k}^{(l)}.
\end{equation}

\subsubsection{Feed-Forward Network}

We apply an FFN to each node independently:
\begin{align}
    \mathbf{z}_j^{(l+1)} &= \text{ReLU}\left( \mathbf{W}_2 \cdot \text{ReLU}(\mathbf{W}_1 \cdot \tilde{\mathbf{h}}_j^{(l+1)} + \mathbf{b}_1) + \mathbf{b}_2 \right), \\
    \mathbf{h}_j^{(l+1)} &= \text{LayerNorm} \left( \mathbf{h}_j^{(l)} + \mathbf{z}_j^{(l+1)} \right),
\end{align}
where \( \mathbf{W}_1 \in \mathbb{R}^{d \times d_{\text{ff}}}, \mathbf{W}_2 \in \mathbb{R}^{d_{\text{ff}} \times d} \).

We perform similar transformations to obtain the updated edge feature \( \mathbf{e}_{k \rightarrow j}^{(l+1)} \) from \( \tilde{\mathbf{e}}_{k \rightarrow j}^{(l+1)} \).

\subsection{Training Objectives of Graph Embedder}

The detail of our proposed three categories of self-supervision tasks are listed as below:
\begin{itemize}
    \item \textbf{Geometric Task — Distance Regression.}  
    For each pair of joints \( (j, k) \), we regress the offset vector from joint \( j \) to joint \( k \). Let \( \mathbf{h}_j \in \mathbb{R}^d \) and \( \mathbf{h}_k \in \mathbb{R}^d \) be the node embeddings after the final graph encoder layer. We predict:
    \begin{equation}
        \hat{\mathbf{o}}_{jk} = \mathbf{W}_{\text{geo}} \cdot \text{ReLU}(\text{FFN}_{\text{geo}}([\mathbf{h}_j \| \mathbf{h}_k])) \in \mathbb{R}^3,
    \end{equation}
    where \( [\mathbf{h}_j \| \mathbf{h}_k] \) denotes the concatenation of embeddings from joint \( j \) and joint \( k \). We apply an \( \ell_2 \) loss to the ground-truth offset vector \( \mathbf{o}_{jk} = \mathbf{p}_k - \mathbf{p}_j \):
    \begin{equation}
        \mathcal{L}_{\text{geo}} = \sum_{(j, k) \in \mathcal{P}} \left\| \hat{\mathbf{o}}_{jk} - \mathbf{o}_{jk} \right\|_2^2,
    \end{equation}
    where \( \mathcal{P} \) is the set of joint pairs for which offset regression is performed.

    \item \textbf{Topological Task — LCA Prediction.}  
    Given two nodes \( (j_1, j_2) \), the model predicts their Least Common Ancestor (LCA) in the skeletal tree. Due to symmetry \( \text{LCA}(j_1, j_2) = \text{LCA}(j_2, j_1) \), we first compute:
    \begin{align}
        \mathbf{q}_{j_1 j_2} &= \text{FFN}_{\text{query}}(\mathbf{h}_{j_1} + \mathbf{h}_{j_2}) \in \mathbb{R}^d, \\
        \mathbf{k}_j &= \text{FFN}_{\text{key}}(\mathbf{h}_j) \in \mathbb{R}^d,\quad \forall j \in \mathcal{J}.
    \end{align}
    We compute LCA probabilities using dot-product attention:
    \begin{equation}
        p_j = \frac{\exp(\mathbf{q}_{j_1 j_2}^\top \mathbf{k}_j)}{\sum_{j' \in \mathcal{J}} \exp(\mathbf{q}_{j_1 j_2}^\top \mathbf{k}_{j'})},
    \end{equation}
    and apply a cross-entropy loss with the ground-truth LCA label \( j^\star \):
    \begin{equation}
        \mathcal{L}_{\text{lca}} = - \log p_{j^\star}.
    \end{equation}

    \item \textbf{Semantic Task — Contrastive Joint Name Matching.}  
    We encourage node features to retain semantic consistency with their joint names. Let \( \mathbf{h}_j \in \mathbb{R}^d \) be the node embedding, and let \( \mathbf{e}_j^{\text{text}} \in \mathbb{R}^{512} \) be the CLIP-encoded joint name. We pass both through learnable projections:
    \begin{align}
        \mathbf{z}_j^{\text{node}} &= \text{FFN}_{\text{node}}(\mathbf{h}_j) \in \mathbb{R}^d, \\
        \mathbf{z}_j^{\text{text}} &= \text{FFN}_{\text{text}}(\text{CLIP}(j\text{-name})) \in \mathbb{R}^d.
    \end{align}
    Then we compute the InfoNCE loss across all joints in a batch:
    \begin{equation}
        \mathcal{L}_{\text{sem}} = - \sum_{j} \log \frac{\exp(\text{sim}(\mathbf{z}_j^{\text{node}}, \mathbf{z}_j^{\text{text}})/\tau)}{\sum_{j'} \exp(\text{sim}(\mathbf{z}_j^{\text{node}}, \mathbf{z}_{j'}^{\text{text}})/\tau)},
    \end{equation}
    where \( \text{sim}(\cdot,\cdot) \) is cosine similarity and \( \tau \) is the temperature.
\end{itemize}

\textbf{Final Loss.} The overall training loss is the weighted sum:
\begin{equation}
    \mathcal{L}_{\text{graph}} = \lambda_{\text{geo}} \mathcal{L}_{\text{geo}} + \lambda_{\text{lca}} \mathcal{L}_{\text{lca}} + \lambda_{\text{sem}} \mathcal{L}_{\text{sem}},
\end{equation}
where \( \lambda_{\text{geo}}, \lambda_{\text{lca}}, \lambda_{\text{sem}} \) are tunable hyperparameters.

\subsection{Proof of Theorem}

\begin{theorem*}
  If a model can correctly determine the LCA for any pair of nodes $(i, j)$ in a tree, then the entire tree topology can be uniquely reconstructed.
\end{theorem*}
  
\begin{proof}[Constructive proof]
We describe an explicit reconstruction procedure that uses only LCA queries.

\textbf{Step 1: Find the root.}
Scan all nodes and pick the unique node $r$ such that for every other node $v$, $\mathrm{LCA}(r,v)=r$. 
A non-root node $u$ fails this check because $\mathrm{LCA}(u,\text{parent}(u))=\text{parent}(u)\neq u$.
Thus $r$ is identified.

\textbf{Step 2: Split into the root’s child subtrees.}
Consider all nodes except $r$. 
Place two nodes $u$ and $v$ into the same group iff $\mathrm{LCA}(u,v)\neq r$.
Nodes from different child subtrees of $r$ have LCA $=r$, so they fall into different groups; 
nodes from the same child subtree never produce $r$ as their LCA, so they fall into the same group.
Hence each group is exactly one child subtree of $r$.

\textbf{Step 3: Identify each child of the root.}
In each group $C$, find the unique node $c\in C$ such that for all $v\in C$, $\mathrm{LCA}(c,v)=c$.
This $c$ is the root of that group’s subtree, i.e., a direct child of $r$.
Add edge $(r,c)$.

\textbf{Step 4: Recurse.}
Now treat each group $C$ as an independent problem:
use the same two tests inside $C$ (with $\mathrm{LCA}$ restricted to pairs in $C$) to find the local root $c$ of $C$, split $C\setminus\{c\}$ by whether the LCA equals $c$, identify $c$’s children, add edges, and recurse until all groups are singletons.

This procedure terminates after assigning a unique parent to every non-root node, thereby reconstructing all edges. 
Because every step is determined solely by LCA answers and yields a unique outcome (unique root, unique grouping, unique local roots), the recovered tree is unique.
\end{proof}

\subsection{Residual VQ-VAE}

To convert continuous motion features into discrete tokens, we adopt a Residual Vector Quantized Variational Autoencoder (RVQVAE) as the tokenization backend. Compared to single-level VQ-VAE, the residual formulation improves expressiveness without increasing token sequence length, which is essential for high-fidelity motion reconstruction and efficient generation.

Given the encoder output \( \mathbf{Z}_{\text{joint}} \in \mathbb{R}^{B \times T/r \times J \times W} \), we apply residual quantization in \( R \) stages. Each stage learns a separate codebook \( \mathcal{C}^{(r)} = \{\mathbf{c}_k^{(r)}\}_{k=1}^{K} \subset \mathbb{R}^{d} \), where \( r = 1, \dots, R \) and \( K \) is the number of codewords.

The quantization is performed sequentially:

\begin{equation}
    \mathbf{z}^{(0)} = \mathbf{Z}_{\text{joint}}, \mathbf{z}^{(r)} = \mathbf{z}^{(r-1)} - \text{Quantize}\left(\mathbf{z}^{(r-1)}; \mathcal{C}^{(r)}\right).
\end{equation}

The final quantized feature is reconstructed as:
\begin{equation}
    \hat{\mathbf{Z}}_{\text{joint}} = \sum_{r=1}^{R} \text{Quantize}\left(\mathbf{z}^{(r-1)}; \mathcal{C}^{(r)}\right).
\end{equation}

We apply the same process to the global token \( \mathbf{z}_{\text{token}} \) in parallel. Each quantized codeword index \( z_{t,j}^{(r)} \in \{1, \dots, K\} \) is stored as part of the discrete token matrix \( \mathbf{Z} \in \mathbb{N}^{T/r \times R} \) for downstream modeling.

\subsection{Ablation on RVQ Depth}
\label{sec:app_rvq_depth}
Table~\ref{tab:rvq_tokenizer} studies the effect of RVQ depth $R$.
Shallow RVQ (1--2 codebooks) has insufficient capacity, leading to large MPJPE and GeoDist.
Increasing depth to 4--6 substantially improves reconstruction across all datasets.
While RVQ-8 yields slightly lower GeoDist on some splits, it degrades MPJPE on HumanML3D and Zoo,
suggesting diminishing returns and reduced robustness.
We therefore use RVQ-6 as a balanced default throughout the paper.

\begin{table*}[t]
\centering
\caption{\textbf{Ablation on RVQ Tokenizer Depth.}}
\label{tab:rvq_tokenizer}
\scriptsize
\setlength{\tabcolsep}{6pt}
\begin{tabular}{l|c|c|c|c|c|c|c|c|c}
\toprule
\multirow{2}{*}{Method} 
& \multicolumn{3}{c|}{\textbf{MPJPE} $\downarrow$} 
& \multicolumn{3}{c|}{\textbf{MPJPE (no trans)} $\downarrow$} 
& \multicolumn{3}{c}{\textbf{GeoDist} $\downarrow$} \\
\cmidrule(lr){2-4}\cmidrule(lr){5-7}\cmidrule(lr){8-10}
& H3D & Obj-XL & Zoo 
& H3D & Obj-XL & Zoo 
& H3D & Obj-XL & Zoo \\
\midrule
RVQ-1 & 0.3663 & 0.1910 & 0.2200 & 0.1328 & 0.1439 & 0.0799 & 6.41$^\circ$ & 21.89$^\circ$ & 16.71$^\circ$ \\
RVQ-2 & 0.2232 & 0.1586 & 0.3503 & 0.1066 & 0.1237 & 0.0785 & 6.03$^\circ$ & 18.30$^\circ$ & 16.00$^\circ$ \\
RVQ-4 & 0.1052 & 0.1053 & 0.1157 & 0.0647 & 0.0869 & 0.0612 & 4.23$^\circ$ & 13.88$^\circ$ & 14.12$^\circ$ \\
RVQ-6 & 0.1084 & 0.0983 & 0.1008 & 0.0588 & 0.0787 & 0.0635 & 3.96$^\circ$ & 12.12$^\circ$ & 13.88$^\circ$ \\
RVQ-8 & 0.1229 & 0.0904 & 0.1526 & 0.0579 & 0.0684 & 0.0581 & 3.84$^\circ$ & 11.49$^\circ$ & 13.04$^\circ$ \\
\bottomrule
\end{tabular}
\end{table*}

\begin{figure*}[t]
  \centering
  \includegraphics[width=\textwidth]{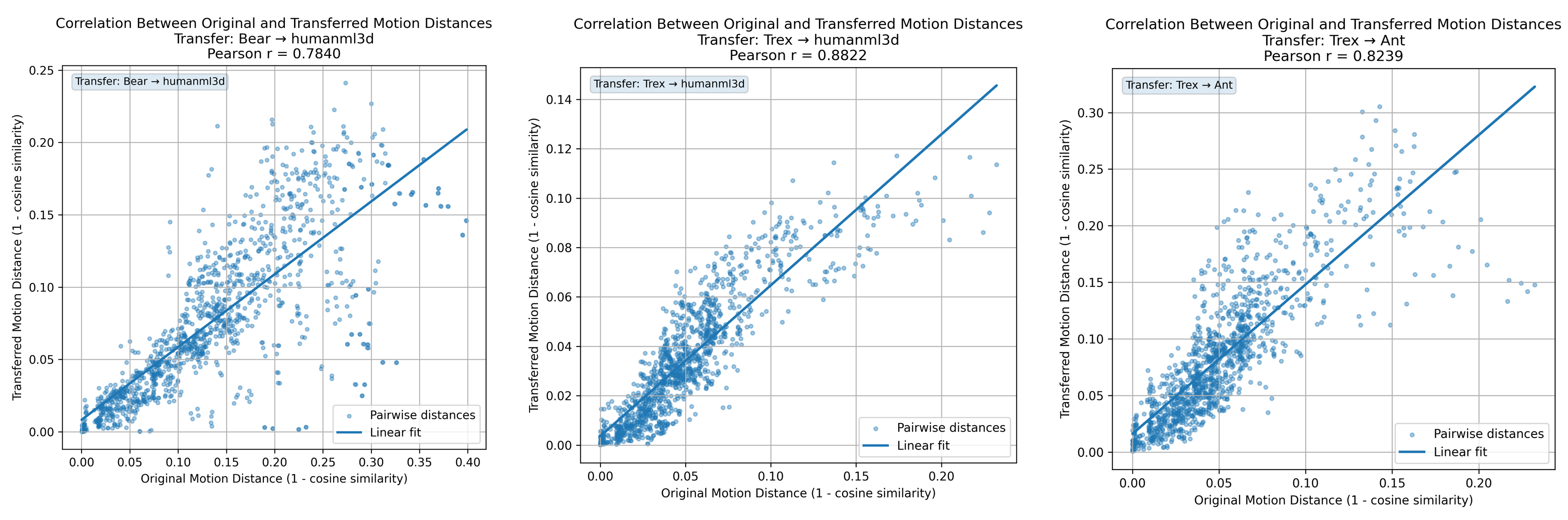}
  \caption{
  \textbf{Cross-skeleton motion distance correlation under topology transfer.}
  Each dot corresponds to a pairwise motion distance computed on the \emph{source} skeleton (x-axis) and the corresponding distance after retargeting to a \emph{target} skeleton (y-axis).
  Strong positive correlations (Pearson $r$) indicate that NEC preserves motion semantics under cross-topology transfer.
  }
  \label{fig:correlation}
\end{figure*}

\section{Evaluation Metrics}
\label{sec:metrics}
{For \textbf{Text-to-Motion (T2M)}, follow existing works \cite{guo2022tm2t, guo2023momask, jiang2024motiongpt, wang2024motiongpt2, zhu2025motiongpt3}, we evaluate motion quality and text–motion alignment. Motion realism is measured by Fréchet Inception Distance (FID), while R-Precision (R@1/2/3) assess semantic consistency between motion and text. }

\textbf{R-Precision.} 
R-Precision measures retrieval performance by computing the fraction of relevant items within the top-$R$ retrieved results. {In text-to-motion, this means retrieving the correct motion from a database given a text query, or vice versa.  }

\begin{equation}
R\text{-Prec} = \frac{| \text{Rel} \cap \text{Top-}R |}{R}
\end{equation}

where $\text{Rel}$ is the set of relevant items (ground-truth matches) and $\text{Top-}R$ is the set of retrieved items at rank $R$. The metric ranges from 0 to 1, with higher values indicating better retrieval accuracy.

\textbf{Fr\'echet Inception Distance (FID).} 
FID~\cite{guo2022generating} measures the distributional distance between real and generated samples in a feature space, capturing both mean and covariance statistics. Lower FID indicates that generated samples are closer to real samples in distribution.  

\begin{equation}
\text{FID} = \| \mu_r - \mu_g \|_2^2 + \mathrm{Tr}\big(\Sigma_r + \Sigma_g - 2(\Sigma_r \Sigma_g)^{1/2}\big)
\end{equation}

where $\mu_r, \Sigma_r$ are the mean and covariance of real samples in the feature space, and $\mu_g, \Sigma_g$ are the corresponding statistics of generated samples. The first term measures the distance between means, while the second term accounts for differences in covariance structure.

\textbf{MPJPE.} Mean Per Joint Position Error (MPJPE) is a widely used metric to evaluate the accuracy of reconstructed 3D skeletons against ground-truth skeletons. It computes the average Euclidean distance between corresponding joints.  
\begin{equation}
\text{MPJPE} = \frac{1}{T \cdot J} \sum_{t=1}^{T} \sum_{j=1}^{J} \left\lVert \hat{\mathbf{x}}_{t,j} - \mathbf{x}_{t,j} \right\rVert_{2},
\end{equation}
where $\hat{\mathbf{x}}_{t,j} \in \mathbb{R}^{3}$ denotes the predicted 3D position of joint $j$ at frame $t$, and $\mathbf{x}_{t,j} \in \mathbb{R}^{3}$ is the corresponding ground-truth joint position. $T$ is the number of frames and $J$ is the number of joints. The Euclidean norm $\lVert \cdot \rVert_{2}$ measures the spatial error per joint, and MPJPE averages this over all joints and frames.  

\textbf{Geodesic Distance.} Geodesic distance is a metric to evaluate the difference between two 3D rotation matrices, often used for skeletal joint rotations. It measures the shortest distance along the manifold of the special orthogonal group $SO(3)$.  
\begin{equation}
\text{Geo}(\hat{R}, R) = \arccos \left( \frac{\text{trace}(\hat{R} R^{\top}) - 1}{2} \right),
\end{equation}
where $\hat{R}, R \in SO(3)$ denote the predicted and ground-truth $3 \times 3$ rotation matrices, and $\text{trace}(\cdot)$ is the matrix trace operator. This formula computes the geodesic angular distance in radians between the two rotations. When averaged across joints and frames, this metric captures how well the predicted rotations align with the ground truth.

\begin{figure*}[t]
\centering
\includegraphics[width=\textwidth]{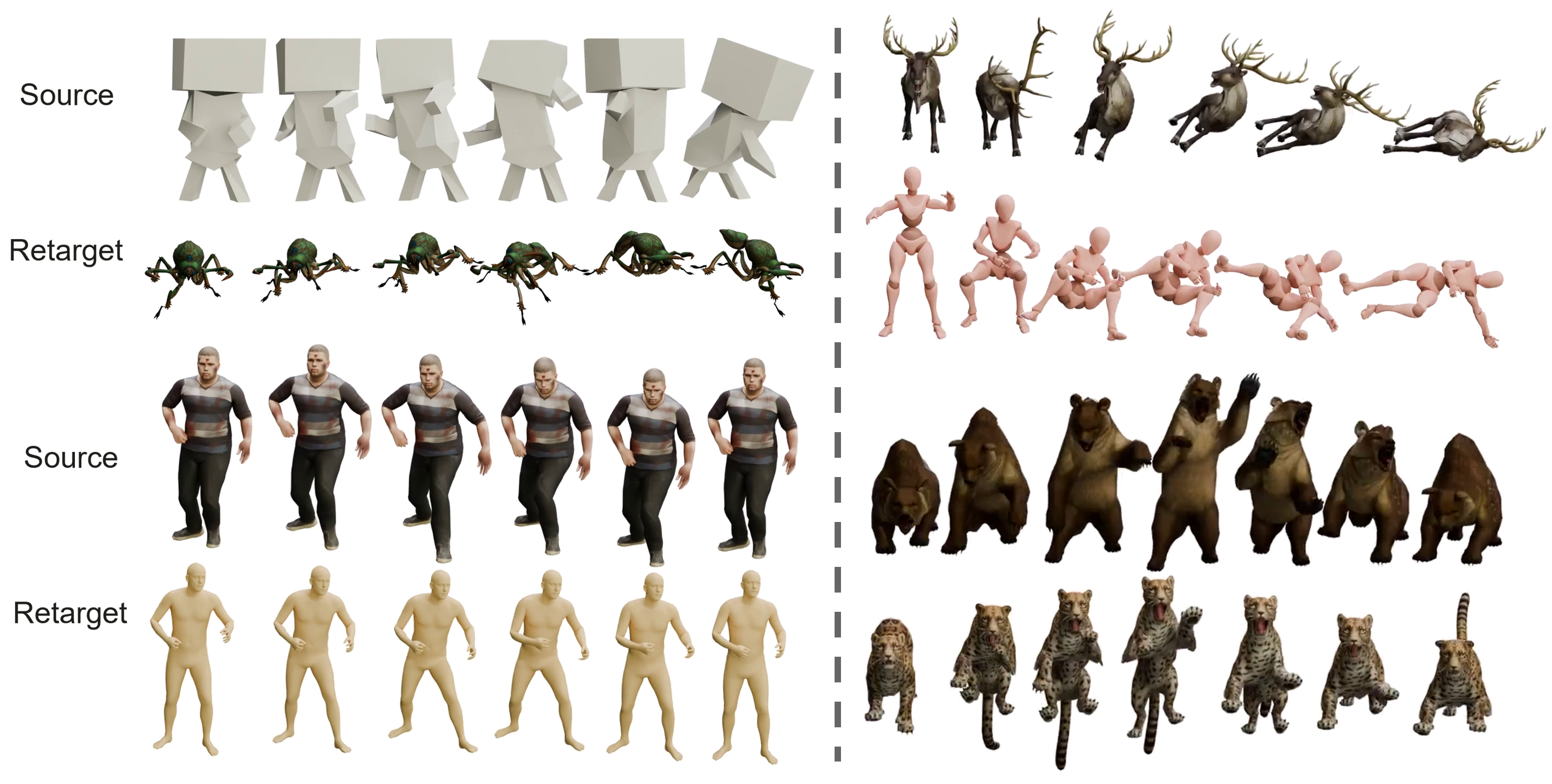}
\caption{Qualitative motion transfer results across different skeletons and object categories.}
\label{fig:qual_token_ops}
\end{figure*}

\section{Motion Transfer Results}
\label{app:trans_res}

\subsection*{Qualitative Results}

We demonstrate \textbf{motion transfer} (source tokens + target OwO) across arbitrary topologies (Fig.~\ref{fig:qual_token_ops}).
\emph{Crucially, this requires no task-specific fine-tuning or auxiliary heads}: because NEC produces \emph{topology-agnostic} tokens and the decoder is \emph{OwO-conditioned}, a single trained model supports \textbf{zero-shot transfer} by decoding source token sequences under different target OwOs (morphology swap).
This enables direct motion transfer between species with diverse skeletons while maintaining temporal coherence.

\subsection*{Quantitative Results}
Table~\ref{tab:ablation_transfer_loss} quantitatively validates the qualitative transfer results in Fig.~\ref{fig:qual_token_ops}.
By transferring 512 randomly sampled motions to target skeletons from different datasets while preserving their original text annotations, we evaluate motion--text transfer retrieval using R-Precision.

Overall, configurations with structural objectives, particularly 0111 (LCA+Dist.+Con.) and 1111, achieve the best transfer performance across datasets and retrieval depths.
These trends closely match the main ablation results in Table~\ref{tab:ablation_loss}, indicating that improvements from structural learning generalize beyond reconstruction to cross-skeleton semantic alignment.
We note that retrieval performance on Zoo is consistently lower, as its text annotations explicitly include species-specific descriptors, making correct retrieval inherently more challenging after cross-species transfer.

The results confirm that topology-agnostic tokens combined with OwO-conditioned decoding enable zero-shot motion transfer that preserves sufficient semantic information for text retrieval, making cross-species motion transfer feasible in practice.

\section{Additional Analysis on Topology Invariance}
\label{sec:app_topology_analysis}

To evaluate whether NEC truly learns a topology-agnostic motion representation—rather than relying on joint-index alignment—we analyze the preservation of pairwise motion distances under cross-skeleton transfer.
For a given source skeleton, we compute pairwise distances between its motion sequences in the learned latent space.
We then retarget the same motions to a different skeleton and recompute the corresponding distances.

Fig.~\ref{fig:correlation} plots the original distances against the transferred distances.
If the representation were tied to a fixed joint ordering or topology, the correlation would collapse.
Instead, we observe strong positive correlations across diverse morphological gaps, including quadruped$\rightarrow$human, bipedal dinosaur$\rightarrow$human, and bipedal dinosaur$\rightarrow$multi-leg insect transfers.

These results indicate that NEC preserves relative motion semantics independently of skeletal topology, providing quantitative evidence for topology-invariant motion encoding without explicit supervision for cross-species transfer.

\section{Limitations and Future Directions}
\label{app:limitations}

\subsection*{Limitations}

\begin{itemize} \item \textbf{Transfer and Generate Performance}: The experimental results on cross-skeleton transfer and motion generation tasks are not particularly impressive. This is primarily because our current work focuses on constructing a unified BVH representation framework and learning topology-agnostic motion representations, while the transfer and generation functionalities serve only as application examples of the proposed framework. Therefore, these limitations should not be overinterpreted as core defects of the entire method, but rather reflect areas for improvement in the specific application implementation.
\item \textbf{Data Dependency}: Like many data-driven approaches, the performance of NECromancer is influenced by the quality and diversity of the training data. While our dataset is carefully constructed to cover a wide range of motions and skeletal configurations, it still provides limited coverage of rare or highly complex topologies, which may affect generalization in these cases.

\item \textbf{Robustness to Skeletal Structure Changes}: Although this method claims to achieve topology-agnostic motion generation, practical applications show that when there are significant differences between source and target skeletons (e.g., quadruped to human conversion), the model may produce inaccurate retargeting results. This indicates that the adaptability to extreme topology variations still needs improvement.

\item \textbf{High Computational Complexity}: The model employs complex graph neural network architectures and multi-stage variational autoencoders (VQ-VAE) for modeling, resulting in lengthy training and inference times, especially when processing large-scale animation sequences. This poses challenges for real-time application scenarios.

\item \textbf{Lack of Explicit Physical Constraint Modeling}: Although unreasonable physical cases are eliminated through preprocessing steps, the model itself does not explicitly incorporate rigid body dynamics or other physical rules to ensure the authenticity and stability of generated motions.

\item \textbf{Semantic Consistency in Text-to-Motion Generation}: In text-driven tasks, despite employing contrastive learning techniques to enhance semantic consistency, the model may still generate actions that do not align with the input text due to the ambiguity and polysemy of natural language descriptions.
\end{itemize}

\subsection*{Future Directions}

\begin{itemize} \item \textbf{Enhanced Cross-Species/Topology Generalization}: Further exploration can be conducted to enable the model to better understand common structural features (such as joint degrees of freedom and movement patterns) across different organisms, thereby improving performance in extreme topology conversions. This could involve incorporating richer prior knowledge, such as skeletal structure classification systems based on evolutionary biology.
\item \textbf{Efficiency and Scalability Optimization}: To address the high computational overhead of the current model, future work can explore lightweight graph neural network architectures and combine knowledge distillation and model compression techniques to reduce deployment costs. Additionally, distributed training strategies can be investigated to support larger-scale datasets.

\item \textbf{Integration of Physics Simulation Mechanisms}: Introducing physics engines as post-processing modules to validate and correct generated motions can ensure compliance with real-world motion laws, thereby enhancing generation quality and credibility.

\item \textbf{Improved Text Understanding Accuracy}: Incorporating more advanced language models (such as GPT-4 or Qwen series) can improve text parsing capabilities, enabling more precise capture of user intentions and contextual information. Furthermore, integrating visual information (such as images or videos) can facilitate multimodal joint modeling.

\item \textbf{Extension to More Animation Content Types}: Future work can expand to other domains such as robot control, and virtual character interactions. This would require customized modeling approaches and evaluation metrics for specific scenarios.

\item \textbf{Introduction of Dynamic Environment Awareness}: The current model primarily focuses on motion generation under static skeletons. Future work can incorporate environmental factors (such as terrain or obstacles) to make generated motions more contextually adaptive, suitable for virtual reality and game development applications.

\end{itemize}

\end{document}